\documentclass[11pt]{article}
\usepackage[margin=1in]{geometry}

\usepackage[utf8]{inputenc} 
\usepackage[T1]{fontenc}    
\usepackage{hyperref}       
\usepackage{url}            
\usepackage{booktabs}       

\usepackage{amsthm}
\theoremstyle{plain}
\newtheorem{theorem}{Theorem}[section]
\newtheorem{proposition}[theorem]{Proposition}
\newtheorem{lemma}[theorem]{Lemma}

\theoremstyle{definition}
\newtheorem{definition}[theorem]{Definition}
\newtheorem{assumption}[theorem]{Assumption}
\theoremstyle{remark}

\usepackage[textsize=tiny]{todonotes}

\usepackage{graphicx}
\usepackage{multirow}
\usepackage{wrapfig}
\usepackage{subcaption}
\usepackage{booktabs}
\usepackage[title]{appendix}

\usepackage{algorithm}
\usepackage{algorithmic}
\usepackage[flushleft]{threeparttable}
\usepackage[utf8]{inputenc} 
\usepackage[T1]{fontenc}    
\usepackage{hyperref}       
\usepackage{url}            
\usepackage{booktabs}       
\usepackage{nicefrac}       
\usepackage{microtype}      
\usepackage{amsthm, amsmath, amsfonts, amssymb}
\usepackage{mathrsfs,mathtools}
\usepackage{bbm}
\usepackage{bm}
\usepackage{graphicx}
\usepackage{xcolor}
\usepackage{stackengine}
\usepackage{nicefrac}       
\usepackage{microtype}      
\usepackage{lipsum}
\usepackage{fancyhdr}       
\usepackage{graphicx}       
\graphicspath{{media/}}     

\usepackage{authblk}

\newcommand{\norm}[1]{\left\lVert#1\right\rVert}

\newcommand*{\ie}{\emph{i.e.}{}}

\title{\bf Counterfactual Fairness Through\\Transforming Data Orthogonal to Bias}

\author[1]{Shuyi Chen}
\author[1]{Shixiang Zhu}
\affil[1]{Carnegie Mellon University}

\begin{document}
\maketitle

\begin{abstract}
Machine learning models have demonstrated exceptional capabilities in solving complex problems across a variety of domains. However, these models can sometimes exhibit biased decision-making, leading to unequal treatment of different groups. 
Despite substantial research on counterfactual fairness, existing methods remain underdeveloped in addressing the impact of multivariate and continuous sensitive variables on decision-making outcomes.
To tackle this gap, we propose a novel data pre-processing algorithm, \emph{Orthogonal to Bias} (\texttt{OB}), which is designed to eliminate the influence of a group of continuous sensitive variables, thereby promoting counterfactual fairness in machine learning applications.
Our approach, based on the assumption of an elliptical distribution within a structural causal model (SCM), shows that counterfactual fairness can be achieved by ensuring the data is orthogonal to the observed sensitive variables.
The \texttt{OB} algorithm is model-agnostic, making it applicable to a wide range of machine learning models and tasks. To enhance numerical stability, we also introduce a sparse variant that incorporates regularization. Empirical evaluations on both simulated and real-world datasets—spanning scenarios with both discrete and continuous sensitive variables—demonstrate that our method effectively promotes fairer outcomes without compromising predictive accuracy.
\end{abstract}

\section{Introduction}
Fairness concern in machine learning has catalyzed a growing body of research aimed at identifying, understanding, and mitigating the biases present in data and algorithms. Among the various conceptual frameworks developed to address this issue~\cite{bolukbasi2016man, hardt2016equality, dwork2012fairness, grgic2016case, zemel2013learning, zafar2017fairness}, the notion of \emph{counterfactual fairness}~\cite{kusner2017counterfactual} emerges as particularly significant. 
It seeks to ensure that a decision for an individual remains consistent with the decision that would have been made in a counterfactual scenario where the individual's sensitive attributes were different. This concept is especially powerful, as it aligns closely with intuitive notions of individual fairness and justice, offering a rigorous standard by which to evaluate and mitigate bias~\cite{10.1145/3597199}.


Significant progress has been made in achieving counterfactual fairness, with extensive studies exploring various approaches~\cite{hardt2016equality, kusner2017counterfactual, wang2019equal, chen2023learning, chiappa2019path}. 
However, existing methods face two significant challenges in practical applications: ($i$) they typically address discrete or binary sensitive variables and struggle to accommodate continuous variables such as age or weight~\cite{grari2023adversarial}; ($ii$) while they can effectively handle a single sensitive variable, they do not extend well to scenarios involving multiple sensitive variables, limiting their practical utility.
These limitations restrict their applicability in many real-world settings where sensitive attributes lie on a spectrum rather than in fixed categories. For example, socio-economic status in employment hiring processes encompasses a range of factors—such as education, income, occupation, neighborhood, and parental education—that cannot be neatly categorized into simple discrete groups, as illustrated in Figure~\ref{fig:motivating-exp}. Training models directly on historical hiring data may inadvertently bias them toward candidates with certain types of experience—factors often correlated with higher socio-economic status~\cite{barocas2016big}. Such bias may arise from unrepresentative training data or from traditional fairness methods that fail to account for the complex influence of continuous, multivariate sensitive variables on decision-making.

\begin{figure}[!t]
     \centering
     \includegraphics[width=0.7\linewidth]{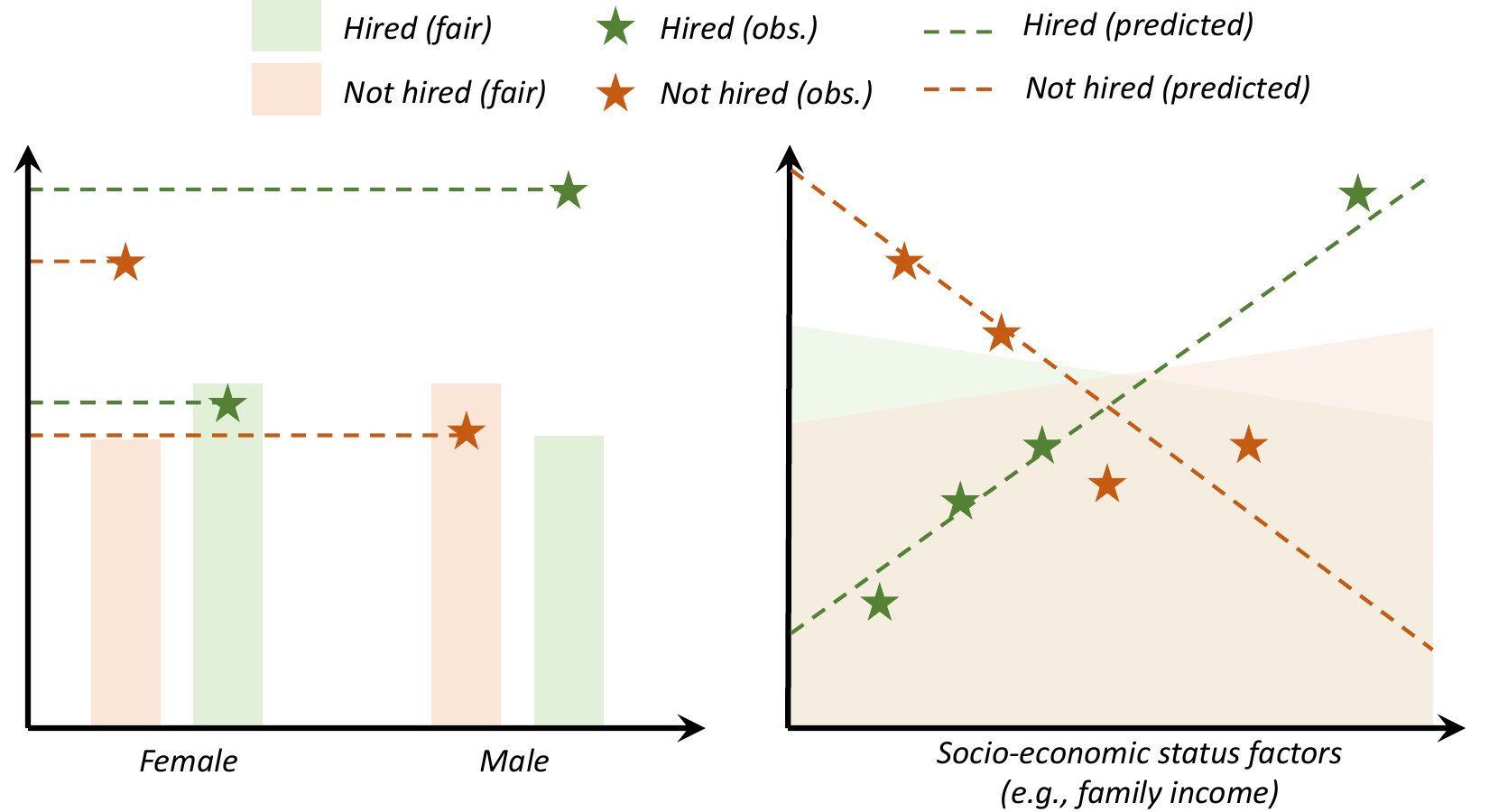}
     \caption{
    This motivating example illustrates the distinction between binary and continuous sensitive variables in the context of hiring decisions. The dashed lines indicate the predicted hiring decisions, and the shaded indicates the fair decisions. The left panel presents a simpler scenario with a binary sensitive variable, such as gender, where adjustments for fairness are more straightforward due to limited, exhaustively enumerable values. In contrast, the right panel delves into the complexity introduced by a continuous sensitive variable, such as family income. In this case, counterfactual estimation becomes significantly more challenging due to potentially unseen or sparsely distributed values. Additionally, socio-economic status, as a multi-dimensional variable combining various factors, adds complexity to establishing fair decision-making processes.
     }
     \label{fig:motivating-exp}
     \vspace{-.12in}
\end{figure}

To tackle these challenges, we propose a novel data pre-processing approach that aims to remove the influence of a group of continuous sensitive variables from the data, thereby promoting counterfactual fairness in downstream machine learning tasks. We begin by proving that counterfactual fairness can be attained asymptotically by ensuring the data is uncorrelated with the group of sensitive variables, under the assumption of an elliptical distribution within a structural causal model (SCM) framework~\cite{pearl2009causality}.
Motivated by this theoretical insight, we consider all sensitive variables collectively and propose a data pre-processing algorithm, referred to as \emph{Orthogonal to Bias} (\texttt{OB}). 
This algorithm is designed for minimal data adjustments to achieve orthogonality between non-sensitive and sensitive data. 
To facilitate numerical stability, we also present a sparse variant, \texttt{SOB}, of this algorithm to handle high-dimensional features. 
Then the transformed data produced by \texttt{OB} or \texttt{SOB} can be readily used as input for machine learning models in downstream tasks without being affected by undesirable biases associated with complex sensitive attributes. Importantly, our proposed approach is model-agnostic and thus compatible with a wide range of machine learning algorithms and tasks.
We evaluate the performance of our method on two synthetic datasets and three real-world datasets, encompassing both continuous and discrete sensitive variable settings. Experimental results show that our approach enables models to achieve fairer outcomes while maintaining predictive accuracy comparable to state-of-the-art fair learning methods. Moreover, our findings demonstrate that the effectiveness of our approach extends beyond the elliptical distribution assumption, indicating its robustness and applicability in broader scenarios.

Our contributions in this work are summarized as follows:
\begin{enumerate}
    \item We show that counterfactual fairness can be achieved by ensuring orthogonality between non-sensitive and sensitive variables under mild conditions;
    \item We propose a model-agnostic data pre-processing algorithm, termed \emph{Orthogonal to Bias} (\texttt{OB}), along with a sparse variant (\texttt{SOB}), which facilitates counterfactual fairness in a wide range of downstream machine learning applications involving continuous, multivariate sensitive attributes;
    \item We demonstrate the effectiveness of our algorithm through empirical evaluations on both synthetic and real-world datasets, showing improved fairness–accuracy tradeoffs compared to existing state-of-the-art methods.
\end{enumerate}

\section{Related Work}
This section begins by examining the fundamental fairness definitions that form the basis of our modeling framework. Following this, we delve into prominent machine learning techniques designed to ensure fairness, with a particular emphasis on counterfactual fairness.

\paragraph{Fairness in machine learning} The pursuit of fair decision-making in machine learning has led to diverse approaches for defining and quantifying fairness. Researchers commonly adopt either observational or counterfactual approaches to formalize fairness.
Observational methods typically characterize fairness through metrics derived from observed data and predicted outcomes~\cite{kamiran2012data,luong2011k,chen2023learning,wei2024assessing}. Metrics such as individual fairness (IF)~\cite{dwork2012fairness}, demographic parity or group fairness~\cite{zemel2013learning,khademi2019fairness} and equalized odds~\cite{hardt2016equality} fall under this category. The key idea for the observational fairness metric is viewing fairness as treating similar individuals or individuals belonging to the same groups similarly. For example, 
equalized odds ensures fairness by requiring that the true positive and false positive rates are equal across different groups~\cite{wang2019equal, hardt2016equality}.

Counterfactual methods, on the other hand, offer a causal perspective on fairness. These approaches determine fairness based on potential changes in outcomes if sensitive variables were altered~\cite{hardt2016equality, kusner2017counterfactual, wang2019equal, chen2023learning, chiappa2019path}. Such methods extend the concept of fairness beyond observable measures. For example, the Equal Opportunity (EO) criterion directly compares the actual and counterfactual decisions for the same individual, assuming the individual had the same non-sensitive attributes, providing a more nuanced assessment than equalized odds~\cite{wang2019equal}.

While integrating observational fairness into machine learning models is relatively straightforward, achieving counterfactual fairness often requires approximations of causal models or counterfactuals, which presents two major challenges~\cite{wang2019equal,mishler2021fairness}. First, the process of counterfactual estimation often compromises the predictive accuracy of models due to the exclusion of related information~\cite{agarwal2018reductions,baharlouei2019r,Berk2017ACF}. For example, the FairLearning algorithm~\cite{kusner2017counterfactual} uses a Markov chain Monte Carlo method to simulate unobserved portions of the causal model, making decisions based solely on variables that are not descendants of sensitive variables. Second, estimating counterfactual distributions for sparse or continuous sensitive variables is difficult and often violates basic causal inference assumptions, namely Positivity~\cite{mishler2021fairness}. This underscores the ongoing challenge of achieving counterfactual fairness with continuous sensitive variables. Our method circumvents these two challenges by ensuring minimal data modification to achieve orthogonality between non-sensitive and sensitive data while maintaining the predictive accuracy of machine learning models. 

\paragraph{Counterfactually fair learning} Counterfactual fairness in machine learning is achieved when a decision for an individual would remain unchanged in a counterfactual world where the individual's sensitive variables differ~\cite{kusner2017counterfactual}. This pursuit has led to the development of diverse strategies to enhance fairness, categorized into three main types: pre-processing~\cite{chen2023learning}, post-processing~\cite{mishler2021fairness}, and in-processing methods~\cite{grari2023adversarial}. 

Our approach is most related to~\cite{chen2023learning}, which is also a pre-processing technique for counterfactual fairness. In~\cite{chen2023learning}, the authors propose two distribution adjustment procedures for making counterfactually fair decisions based on adjusted data. While both procedures remove variables' dependence on sensitive variables under respective conditions, their methods provide no guarantee regarding the scale of modification due to the distribution modification. In contrast, our work introduces an exact approach to solving an optimization problem that guarantees minimal modification to the data while ensuring counterfactual fairness under specific assumptions.
Furthermore,
the method proposed by~\cite{chen2023learning} presuppose sensitive variables to be binary (or categorical) so that they can be easily isolated or adjusted based on empirical probability mass function. It does not address situations involving multivariate or continuous sensitive variable with intricate inter-dependency, whereas our proposed \texttt{OB} algorithm aims to remove the influence of a group of continuous sensitive variables from the data, thereby ensuring counterfactual fairness in subsequent machine learning tasks.

Additionally, our work contrasts with~\cite{grari2023adversarial}, the only other effort to the best of our knowledge that specifically tackles counterfactual fairness with continuous sensitive variables. While we include their method as a compared baseline, their approach relies on a generic loss design that lacks explicit guarantees for counterfactual fairness. In contrast, our method offers a theoretical underpinning in the motivation and design of the method, prioritizing minimal data changes to enhance the balance between accuracy and fairness.

Overall, our emphasis on minimal data modification places our proposed algorithm in a unique position within the widely observed fairness-accuracy spectrum~\cite{agarwal2018reductions, baharlouei2019r, Berk2017ACF}. Through our approach, we aim to capture as much information as possible between the target variable 
$Y$ and the features, including the sensitive features.

\begin{figure*}[!t]
     \centering
     \begin{subfigure}[b]{0.49\linewidth}
         \centering     
         \includegraphics[width=\linewidth]{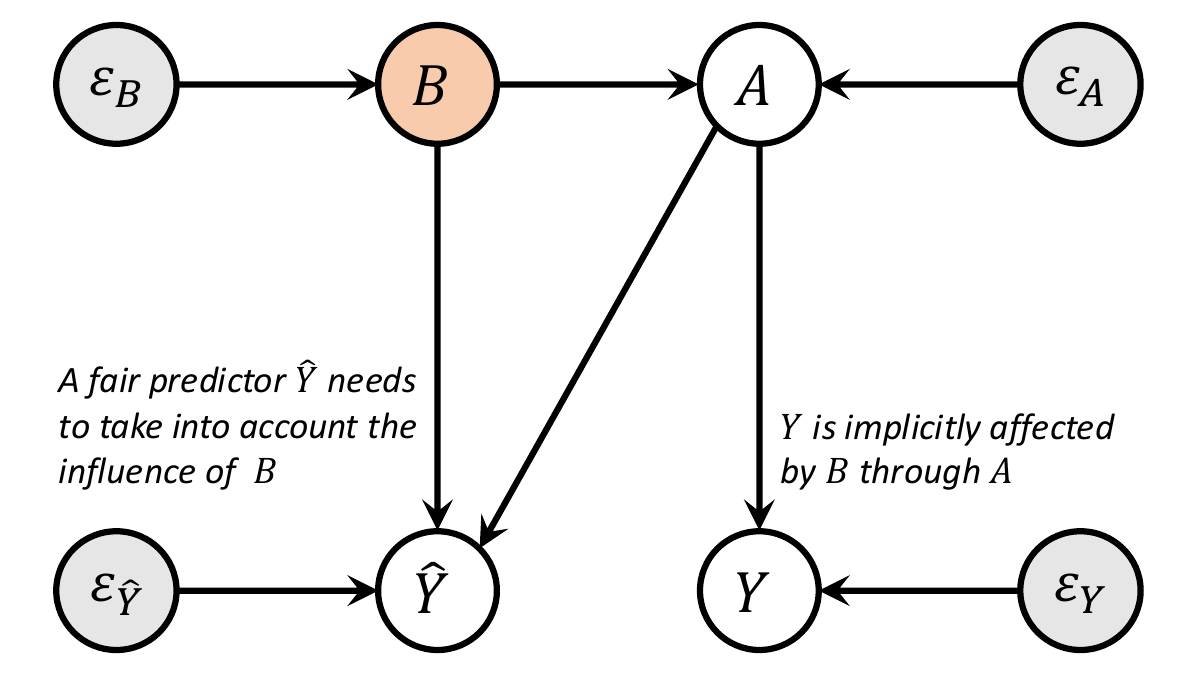}
         \caption{The SCM and typical fair learning methods
         }
         \label{fig:illustration_a}
     \end{subfigure}
     \hfill
     \begin{subfigure}[b]{0.49\linewidth}
         \centering
         \includegraphics[width=\linewidth]{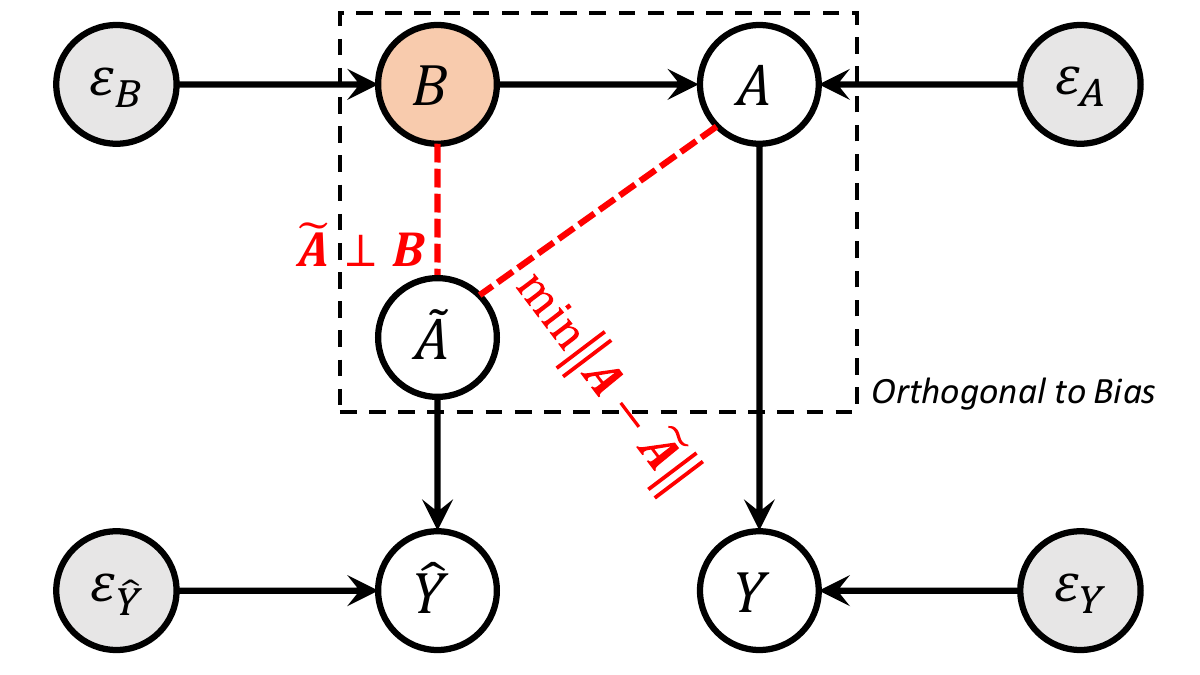}
         \caption{Transforming data with \emph{Orthogonal to Bias} (\texttt{OB})} 
         \label{fig:illustration_c}
     \end{subfigure}
     \caption{Illustration of (a) the structural causal model (SCM) and a common fair learning strategies, as well as (b) the proposed data pre-processing algorithm \emph{Orthogonal to Bias} (\texttt{OB}). 
     The white nodes $A$, $Y$, and $\widehat{Y}$ are the non-sensitive variables, the decision variable, and its prediction, respectively. 
     The red nodes $B$ represent the sensitive variable.  
     The $\widehat{A}$ is the transformed data that is orthogonal to bias in $B$.
     The gray nodes represent exogenous variables. 
     }
     \label{fig:illustration}
\end{figure*}

\section{Methodology}
\label{sec:method}
\subsection{Problem setup}

We jointly define $q$ non-sensitive random variables as $A \in \mathcal{A} \subseteq \mathbb{R}^{q}$, $p$ sensitive variables as $B \in \mathcal{B} \subseteq \mathbb{R}^{p}$, and the decision variable as $Y \in \mathcal{Y}$. The data generation process in our problem setup is described by a Structural Causal Model (SCM)~\cite{pearl2009causality}, as shown in Figure~\ref{fig:illustration_a}. Our setup allows for sensitive variables ($B$) that are both continuous and multivariate, enabling application to a wider range of real-world scenarios involving complex sensitive attributes.

To be specific, we consider the set of endogenous variables $V = \{B, A, Y, \widehat{Y} \}$, where $\{B, A, Y\}$ are the observed variables and $\widehat{Y}$ is the prediction of $Y$ made based on $B$ and $A$. 
We assume that $\varepsilon_B$, $\varepsilon_A$, and $\varepsilon_Y$, which are the exogenous variables that affect $B$, $A$, and $Y$, respectively, are independent of each other.
The structural equations are described with the functions 
$F = \{f_Y, f_A, f_B\}$, one for each component in $V$, detailed as follows:
\begin{equation}
   B = f_B\left(\varepsilon_B\right),\quad
   A = f_A\left(B, \varepsilon_A\right),\quad
   Y = f_Y\left(A, B, \varepsilon_Y\right).\quad
\label{eq:scm}
\end{equation}
According to the above SCM, the bias present in the sensitive variables $B$ can only transmit to the predictor $\widehat{Y}$ via the non-sensitive variables $A$. 
This means that if there are any differences in the distribution of $A$ conditioned on $B$, the decision variable $\widehat{Y}$ based on $A$ might be unfair. For convenience, we denote the collection of all exogenous variables as $\varepsilon = [\varepsilon_A,\varepsilon_B,\varepsilon_Y]$.

In this paper, we aim to design a predictor $\widehat{Y}$ that achieves counterfactual fairness~\cite{kusner2017counterfactual, chen2023learning} without being influenced by the bias in $B$. 
Formally, the counterfactual fairness in our SCM can be defined as follows, following~\cite{kusner2017counterfactual}:
\begin{definition}[Counterfactual Fairness]
\label{eq:definitoin_fair}
Given a new pair of variables $(\mathbf{b}, \mathbf{a})$, a decision variable $Y$ is considered counterfactually fair if, for any $\mathbf{b}' \in \mathcal{B}$,
\begin{equation}
{Y}_{\mathbf{b}'}(\varepsilon)|\left\{B=\mathbf{b}^*, A=\mathbf{a}^*\right\} \stackrel{d}{=} {Y}_{\mathbf{b}^*}(\varepsilon)|\left\{B=\mathbf{b}^*, A=\mathbf{a}^*\right\}, 
\end{equation}
where $P \stackrel{d}{=} Q$ indicates that random variables $P$ and $Q$ are equal in distribution, and ${Y}_{\mathbf{b}}(\varepsilon)$ represents the counterfactual outcome of $Y$ when $B=\mathbf{b}$.
\end{definition}

The above definition implies that the distribution of the counterfactual result should not depend on the sensitive variables conditional on the observed data.
Note that although Definition~\ref{eq:definitoin_fair} uses the decision variable $Y$, it also applies to its predictor $\widehat{Y}$ without any loss of generality~\cite{chen2023learning}. 


\subsection{Achieving counterfactual fairness via data decorrelation}
\label{subsec:debiasing}


To clarify and streamline the presentation of our findings, we begin by establishing that counterfactual fairness can be achieved asymptotically when the sensitive and non-sensitive variables are jointly elliptically distributed and uncorrelated.

Consider a dataset $\mathcal{D} = \{(\mathbf{b}_i, \mathbf{a}_{i}, y_i)\}_{i=1}^{n}$ with $n$ observed data tuples, where $\mathbf{b}_i$, $\mathbf{a}_i$, and $y_i$ represent the $i$-th observation of the sensitive, non-sensitive, and decision variables, respectively.  
We use $\mathbf{A} = [\mathbf{a}_1, \dots, \mathbf{a}_n]^\top \in \mathbb{R}^{n \times q}$ to denote the data matrix of non-sensitive variables $A$, and use $\mathbf{B} = [\mathbf{b}_1, \dots, \mathbf{b}_n]^\top \in \mathbb{R}^{n \times p}$ to denote the data matrix of sensitive variables $B$ in dataset $\mathcal{D}$.

To build intuition, note that counterfactual fairness is satisfied if the distribution of $A$ remains unchanged under interventions on $B$ by Definition \ref{eq:definitoin_fair}. In such cases, altering the value of $B$ does not change $A$, and thus any predictor trained on $A$ will produce the same output under different counterfactual settings of $B$.
The following proposition formalizes this intuition in the context of elliptical distributions and shows that uncorrelatedness between $A$ and $B$ is sufficient to ensure counterfactual fairness asymptotically.

We first introduce the following assumption:
\begin{assumption}[Elliptical Distribution]
\label{as:elliptical}
Let the sensitive variable $B$ and non-sensitive variable $A$ be jointly distributed as an elliptically contoured distribution:
\[
(A, B) \sim \mathcal{E}_{q+p}(\boldsymbol{\mu}, \Sigma, \psi),
\]
where $\boldsymbol{\mu} = [\boldsymbol{\mu}_A^\top, \boldsymbol{\mu}_B^\top]^\top$ is a location vector, $\Sigma$ is a positive-definite scatter matrix, and $\psi\!:\![0,\infty)\!\to\![0,\infty)$ is a characteristic generator.
\end{assumption}

Building on this assumption, we present the following theorem:

\begin{theorem}
\label{thm:elliptical}
Under Assumption~\ref{as:elliptical}, any predictor $\widehat{Y}$ trained solely on $A$ is asymptotically counterfactually fair if $\Sigma_{AB}=0$, provided that $\widehat{Y}$ is $\ell_1$-consistent and converges to a mean-functional of $A$. That is,
\[
\widehat{Y}_{B=\mathbf{b}'} \mid \left\{B=\mathbf{b}^*,A=\mathbf{a}^* \right\}
\xrightarrow{p}
\widehat{Y}_{B=\mathbf{b}^*} \mid \left\{B=\mathbf{b}^*,A=\mathbf{a}^*\right\}.
\]
\end{theorem}

\begin{proof}
In the following proof, we focus on the case where the predictor $\widehat{Y}$ is binary. This simplification enables us to establish fairness by demonstrating the equivalence of expected outcomes, which, in the case of a Bernoulli random variable, also implies equivalence in distribution \cite{chen2023learning}.

We begin by analyzing the conditional distribution of $A$ under Assumption~\ref{as:elliptical}. 
If $(A, B) \sim \mathcal{E}_{q+p}(\boldsymbol{\mu}, \Sigma, \psi)$ has finite second moments, then the conditional expectation of $A$ given $B$ is linear:
\[
\mathbb{E}[A \mid B] =
\boldsymbol{\mu}_A + \Sigma_{AB} \Sigma_{BB}^{-1}(B - \boldsymbol{\mu}_B).
\]
When $\Sigma_{AB} = 0$, this reduces to
\[
\mathbb{E}[A \mid B = \mathbf{b}] = \boldsymbol{\mu}_A,
\]
so the conditional mean of $A$ does not vary with $B$. 
As a result, for any function $g$ whose output depends only on the first moment of $A$—that is, any mean-functional such as a affine function, or more generally any function learned from a sufficiently large sample of $A$ alone—the conditional expectations are equal:
\begin{equation}
\label{eq:square_int}  
\mathbb{E}[g(A) \mid B = \mathbf{b}'] = \mathbb{E}[g(A) \mid B = \mathbf{b}^*].
\end{equation}

Now let $\{g_n\}$ be a sequence of predictors $g$ trained solely on $A$ from $n$ samples, such that
\[
\mathbb{E}\left[g_n(A) - g_\infty(A)\right] \to 0,
\]
that is the learner is $\ell_1$-consistent. Define $\widehat{Y}_n \coloneqq g_n(A)$ and $\widehat{Y}_\infty \coloneqq g_\infty(A)$, and let $A_{\mathbf{b}} \coloneqq f_A(\mathbf{b}, \varepsilon_A)$ denote the counterfactual realization of $A$ under $B = \mathbf{b}$.

We define the counterfactual discrepancy as follows:
\[
\Delta_n \coloneqq 
\mathbb{E}[\widehat{Y}_{n, B = \mathbf{b}'} \mid B = \mathbf{b}^*, A = \mathbf{a}^*]
-
\mathbb{E}[\widehat{Y}_{n, B = \mathbf{b}^*} \mid B = \mathbf{b}^*, A = \mathbf{a}^*].
\]
Using triangular inequality, we have:
\begin{align*}
|\Delta_n| 
&= \left| \mathbb{E}\left[ g_n(A_{\mathbf{b}'}) - g_n(A_{\mathbf{b}^*}) \right] \right| \\
&\le \left| \mathbb{E}\left[ g_n(A_{\mathbf{b}'}) - g_\infty(A_{\mathbf{b}'}) \right] \right| 
+ \left| \mathbb{E}\left[ g_\infty(A_{\mathbf{b}'}) - g_\infty(A_{\mathbf{b}^*}) \right] \right| 
+ \left| \mathbb{E}\left[ g_\infty(A_{\mathbf{b}^*}) - g_n(A_{\mathbf{b}^*}) \right] \right|.
\end{align*}

The first and last terms converges to zero in probability due to $\ell_1$-consistency. 
The second term vanishes exactly by \eqref{eq:square_int}. Therefore, $|\Delta_n| = o_p(1)$, which implies
\[
\mathbb{E}\left[\widehat{Y}_{B = \mathbf{b}'} \mid B = \mathbf{b}^*, A = \mathbf{a}^*\right]
=
\mathbb{E}\left[\widehat{Y}_{B = \mathbf{b}^*} \mid B = \mathbf{b}^*, A = \mathbf{a}^*\right]
+ o_p(1),
\]
completing the proof.
\end{proof}

Theorem~\ref{thm:elliptical} suggests that counterfactual fairness in the predictor $\widehat{Y}$ can be achieved by decorrelating the non-sensitive variables $A$ from the sensitive variables $B$. This insight motivates the development of our data pre-processing algorithm, which aims to minimally adjust the observed data in order to remove correlation between $A$ and $B$ while preserving predictive accuracy.

It is important to emphasize that while the exact elliptical distribution may not always hold in practice, empirical results support the robustness of our approach. As demonstrated in Section~\ref{sec:exp}, our method performs reliably across a variety of real-world datasets, including those with categorical and continuous sensitive attributes, even when the strict elliptical assumption is violated.
In addition, we extend this fairness guarantee to general predictors $\widehat{Y}$, making our approach model-agnostic under the stronger assumption that $(A, B)$ are jointly normally distributed, as detailed in Appendix~\ref{app:joint-normal-proof}.

\subsection{Orthogonal to bias}
\label{sec:og}

\begin{algorithm}[!t]
    \caption{Sparse Orthogonal to Bias (\texttt{SOB})}
    \label{alg:og}
\begin{algorithmic}[1]
    \STATE \textbf{Input}: Non-sensitive data $\mathbf{A}$, sensitive data $\mathbf{B}$, rank $k$, convergence threshold $\eta$, sparsity threshold $h$.
    \STATE Standardize $\mathbf{A}$ and $\mathbf{B}$.
    \FOR{$j = 1, \ldots, k$} 
    \STATE Initialize $\mathbf{u}_j^{(0)}$ randomly, set $\mathbf{s}_j^{(0)}=0$, $t=1$, $\theta = 0$;
    \STATE Define projection matrix: $\mathbf{P}_{j-1} = \mathbf{I} - \sum_{l=1}^{j-1} \mathbf{s}_l \mathbf{s}_l^\top$;

    \WHILE{$\|\mathbf{u}_j^{(t)} - \mathbf{u}_j^{(t-1)}\|_2 > \eta$ and $\|\mathbf{s}_j^{(t)} - \mathbf{s}_j^{(t-1)}\|_2 > \eta$}
    \STATE Compute weight factor:
    \begin{equation*}
        \beta_j^{(t)} \leftarrow (\mathbf{B}^\top \mathbf{B})^{-1} \mathbf{B}^\top \mathbf{P}_{j-1} \mathbf{A} \mathbf{u}^{(t-1)}_j;
    \end{equation*}
    \STATE Update $\mathbf{s}_j$:
    \begin{equation*}
        \mathbf{s}_j^{(t)} \leftarrow \frac{\mathbf{P}_{j-1} \mathbf{A} \mathbf{u}_j^{(t-1)} - \beta_j^{(t)} \mathbf{B}}{\left\| \mathbf{P}_{j-1} \mathbf{A} \mathbf{u}_j^{(t-1)} - \beta_j^{(t)} \mathbf{B} \right\|_2};
    \end{equation*}
    \STATE Update $\mathbf{u}_j$ with sparsity control:
    \begin{equation*}
        \mathbf{u}_j^{(t)}  \leftarrow \frac{\mathcal{S}_\theta\left(\mathbf{A}^\top \mathbf{s}_j^{(t)}\right)}{\norm{\mathcal{S}_\theta\left(\mathbf{A}^\top \mathbf{s}_j^{(t)}\right)}_2},
    \end{equation*}  
    \STATE where $
        \mathcal{S}_\theta(x) = \operatorname{sign}(x)(|x|-\theta) \mathbf{1}(|x| \geq \theta);
    $
    \IF{ $\left\|\mathbf{A}^\top \mathbf{s}_j^{(t)}\right\|_1 > h$ }
    \STATE Adjust $\theta$ such that $\left\| \mathbf{u}_j^{(t)} \right\|_1 = h$;
\ELSE
    \STATE Set $\theta \leftarrow 0$;
\ENDIF
    \STATE $t \leftarrow t + 1$
    \ENDWHILE
    \ENDFOR
    \STATE Compute transformed attributes:
    \begin{equation*}
        d_j = \mathbf{s}_j^\top \mathbf{A} \mathbf{u}_j, \quad \mathbf{S}^* = [d_1 \mathbf{s}_1, \ldots, d_k \mathbf{s}_k], \quad \mathbf{U}^* = [\mathbf{u}_1, \ldots \mathbf{u}_k];
    \end{equation*}
    \STATE Calculate processed data: $        \widetilde{\mathbf{A}} = \mathbf{S}^* \mathbf{U^*}^\top$;
    \STATE \textbf{Output}: $\widetilde{\mathbf{A}}$.
\end{algorithmic}
\end{algorithm}

Now we introduce our data pre-processing algorithm, termed as \emph{Orthogonal to Bias} (\texttt{OB}). 
We assume both non-sensitive variables $A$ and sensitive variables $B$ have a zero mean for each column.
This allows us to estimate the covariance between $A$ and $B$ by calculating their inner product:
\[
\widehat{\Sigma}_{AB}
   \;\coloneqq\;
   \frac{1}{n}\,\mathbf A^{\top}\mathbf B.
\]
Since two variables are uncorrelated if and only if their covariance is zero, enforcing $\mathbf{A}^\top \mathbf{B} = 0$ 
ensures that $\mathbf{A}$ and $\mathbf{B}$ are orthogonal and thus uncorrelated.

The key idea of \texttt{OB} algorithm is to adjust the observed non-sensitive data $\mathbf{A}$ in such a way that it is orthogonal to the observed sensitive data $\mathbf{B}$, while ensuring minimal changes to non-sensitive data $\mathbf{A}$. 
Specifically, 
we follow the idea of Orthogonal to Groups introduced by~\cite{aliverti2021removing}, and define a rank $k$ approximation of $\mathbf{A}$ as $\widetilde{\mathbf{A}} = \mathbf{S}\mathbf{U}^\top$, where $\mathbf{U} 
= [\textbf{u}_1, \dots, \textbf{u}_k]$ is a $q \times k$ orthonormal matrix and $\mathbf{S} = [s_{ij}]$ is a $n \times k$ matrix with associated scores.
The goal is to find a transformed $n \times q$ matrix $\widetilde{\mathbf{A}}$ that is orthogonal to $\mathbf{B}$ with minimal change to matrix, as measured by the Frobenius norm  $\|X\|_{\mathrm{F}}=\sqrt{\sum_i \sum_j x_{i j}^2}$.
Formally, we aim to solve the following constrained optimization problem:
\begin{equation} 
\min_{\mathbf S,\mathbf U}\;
    \|\mathbf A-\mathbf S\mathbf U^{\top}\|_{F}^{2}
\quad\text{s.t.}\quad
    \mathbf B^{\top}(\mathbf S\mathbf U^{\top})=\mathbf 0,
    \;
    \mathbf{U} \in \mathcal{G}_{q,k},
\label{eqn:OG} 
\end{equation}
where $\mathbf{0}$ is a $p \times q$ zero matrix and $\mathcal{G}_{q,k}$ denotes the Grassmann manifold of $k$-dimensional subspaces in $\mathbb{R}^q$~\cite{james1954normal}.
The last constraint helps prevent degeneracy, such as basis vectors becoming identically zero or encountering solutions with double multiplicity. Here we let $\mathbf{U}$ be orthonormal matrix in practice, \ie, $\mathbf{U}^\top \mathbf{U}=\mathbf{I}_k$. 
In the following, we focus
on a univariate $\mathbf{b} \in \mathbb{R}^{n}$ for clarity. We note the procedure can be easily extended for multivariate $B$~\cite{aliverti2021removing}. The above constrained optimization \eqref{eqn:OG} can be reformulated in terms of Lagrange multipliers $\boldsymbol{\lambda}=(\lambda_{1},\dots,\lambda_{k})$:
\begin{equation}
\mathcal{L}(\mathbf{S},\mathbf{U},\boldsymbol{\lambda})
    \;=\;
    \bigl\|\mathbf{A}-\mathbf{S}\mathbf{U}^{\top}\bigr\|_{F}^{2}
    +\frac{2}{n}\sum_{j=1}^{k}\lambda_j\sum_{i=1}^{n}s_{ij}{b}_i,
\quad
\label{eqn:OG_2}
\end{equation}
where $\lambda_j \in \mathbb{R}$ denotes the Lagrange multiplier and $b_i$ is the $i$-th item of $\mathbf{b}$. Here, the factor $2/n$ is introduced to simplify the expression for the optimal solutions. It ensures that the first-order condition for \eqref{eqn:OG_2} with respect to $s_{ij}$ involves a common factor of $2/n$, which can then be canceled out during computations.
Let $\mathbf{S}^*$ and $\mathbf{U}^*$ denote the optimal solutions of $\mathbf{S}$ and $\mathbf{U}$, respectively. 
The reformulated equation \eqref{eqn:OG_2} gives a closed-form solution:
\begin{equation}
\lambda^*_j = \frac{\mathbf{b}^\top \mathbf{A} \mathbf{u}^{*}_j}{{\mathbf{b}}^\top \mathbf{b}}, \quad \mathbf{U}^* = \left[\textbf{u}^*_1, \dots, \textbf{u}^*_k\right], \quad \mathbf{S}^* = [s^*_{ij}],
\label{eqn:og_solution}
\end{equation}
where $\mathbf{U}^*$ consists of the first $k$ right singular vectors of $\mathbf{A}$ and $s_{ij}^* = \mathbf{a}^\top_i \mathbf{u}^{*}_j-\lambda^*_j {b}_i$.
The pre-processed non-sensitive data matrix can be obtained by $\widetilde{\mathbf{A}} = \mathbf{S}^*\mathbf{U^*}^\top$. 
Detailed derivations of the closed-form solution \eqref{eqn:og_solution} can be found in Appendix~\ref{app:ob_closed_form}. 

We note that the reconstruction error of $\widetilde{\mathbf{A}}$ equals to that of the standard Singular Value Decomposition (SVD) of $A$ when the sensitive and non-sensitive variables are uncorrelated.
More generally, the additional error introduced by \texttt{OB} is proportional to the degree of collinearity between the subspace spanned by $\mathbf{B}$ and the left singular vectors of $\mathbf{A}$, as formalized in the following lemma~\cite{aliverti2021removing}:
\begin{lemma}
Let $\widehat{\mathbf{A}} = V_kD_kU_k^\top$ denote the best rank-k approximation of the matrix $\mathbf{A}$ obtained from
the truncated SVD of rank $k$. Let $[P]_{ii'} = 1/n + (b_i b_{i'}) / \sum_{i=1}^n b_i^2$. The additional reconstruction error of the \texttt{OB} algorithm compared to SVD is  $\|kPV_kD_k\|_{\mathrm{F}}$.
\end{lemma}
Thus, if $\mathbf{B}$ has minimal correlation with the singular vectors of $\mathbf{A}$, the performance loss due to \texttt{OB} is small, making it a practical method for fair representation learning.

\paragraph{Sparse orthogonal to bias (\texttt{SOB})}
When the number of features $p$ exceeds the number of observations $n$, estimating a low-dimensional structure from high-dimensional data can become numerically unstable~\cite{zou2006sparse}. 
To address this challenge, we introduce a sparse variant of the $\texttt{OB}$ algorithm, referred to as $\texttt{SOB}$.
The $\texttt{SOB}$ imposes an $\ell_1$-norm penalty for $\mathbf{U}$ to encourage sparsity and improve numerical stability, 
in addition to the orthogonality constraints in \eqref{eqn:OG}.
We define $h$ as the $\ell_1$ constraint on $u$. 
Details of the formulation of $\texttt{SOB}$ can be found in Appendix~\ref{append:sob}.
Following derivation by~\cite{aliverti2021removing, witten2009penalized}, Algorithm~\ref{alg:og} illustrates the key steps to implement the $\texttt{SOB}$ algorithm, where $\eta$ represents the minimum change to terminate the iterative optimization process.

Note that with the additional regularization constraints, the solution favors sparsity while satisfying the orthogonal constraint.
Therefore, the \texttt{SOB} also achieves counterfactual fairness under SCM framework with Theorem~\ref{thm:elliptical}.

\begin{figure}[!t]
\centering
\begin{subfigure}{0.48\linewidth}
    \centering
    \includegraphics[width=\linewidth]{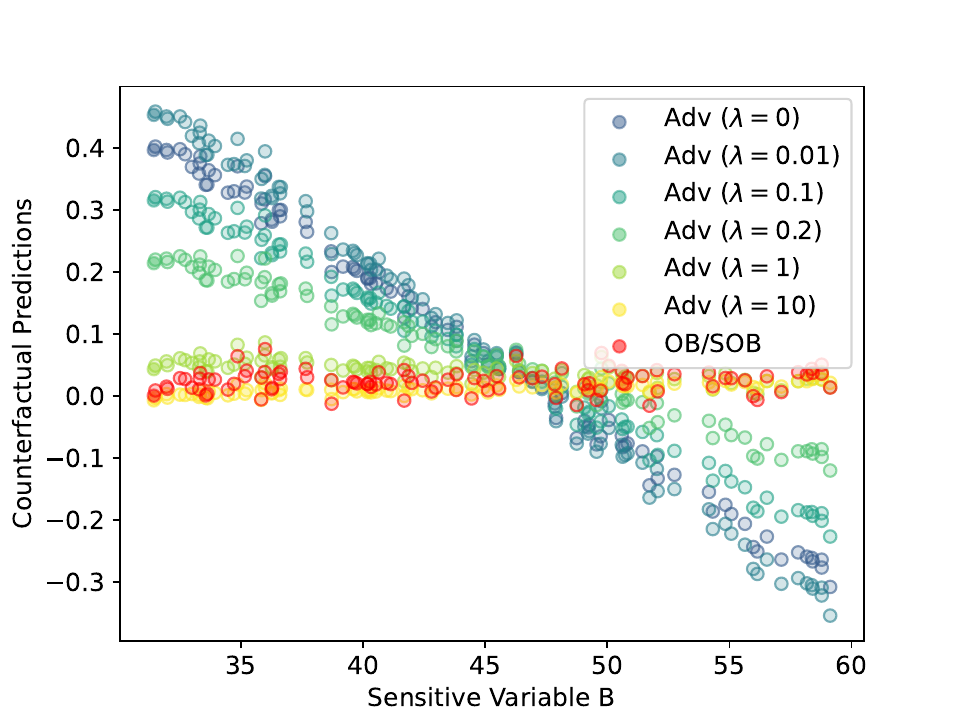}
    \caption{
    Counterfactual prediction of \texttt{OB/SOB} and \texttt{Adv}
    }
    \label{fig:synth_comp_1}
\end{subfigure}%
\hfill \begin{subfigure}{0.48\linewidth}
    \centering
    \includegraphics[width=\linewidth]{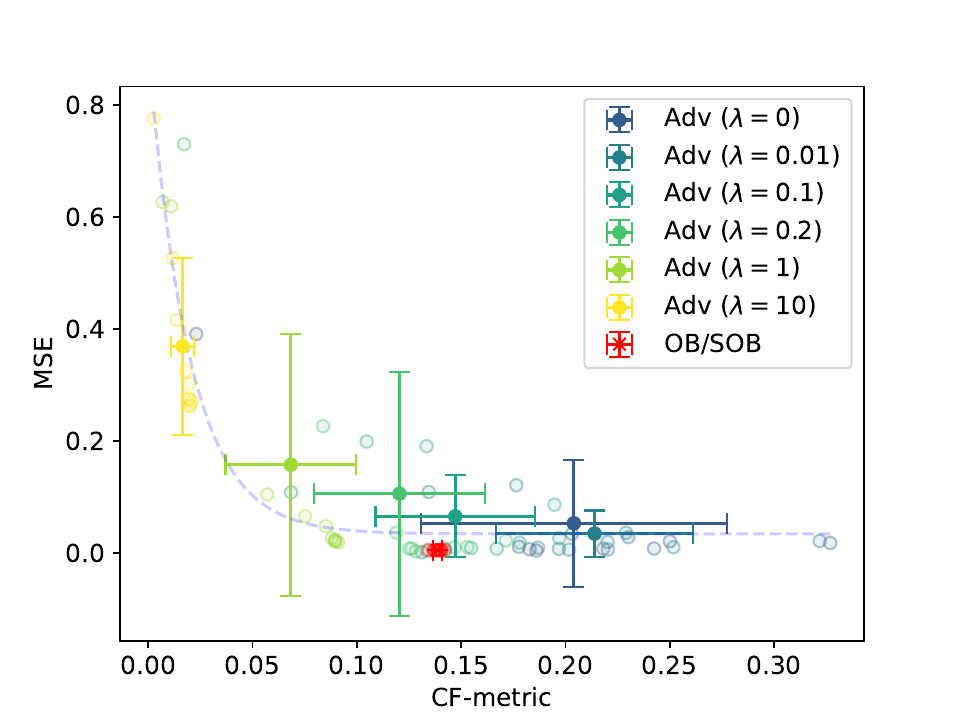}
    \caption{
    Predictive error vs counterfactual fairness
    }
    \label{fig:synth_comp_2}
\end{subfigure}%
\caption{
Comparison of our approach (\texttt{OB}/\texttt{SOB}) with the primary baseline method (\texttt{Adv}) on the synthetic dataset.
(a) Counterfactual predictions for $1,000$ randomly generated counterfactual sample points.
Our method's predictions remain flat across the sensitive variable; (b) Comparison of predictive error and counterfactual fairness between our approach and baseline methods. The dashed line represents an exponential function fitted to the results of \texttt{Adv}. 
The error bars indicate standard deviation. 
Our method is under the curve and achieves a low MSE while maintaining relatively low CF compared to \texttt{Adv}. 
}
\label{fig:synth_comp}
\end{figure}

\section{Experiments}
\label{sec:exp}
\begin{table*}[!t]
    \centering
    \caption{Performance on Datasets with Continuous Sensitive Variables}
    \resizebox{\linewidth}{!}{%
    \begin{tabular}{cccccccccc}
    \toprule
    \multirow{2}{*}{Dataset} & \multirow{2}{*}{Metrics} &
    \multicolumn{2}{c}{Baselines} && 
    \multicolumn{2}{c}{Compared Methods} && \multicolumn{2}{c}{Ours} \\
    \cmidrule{3-4} \cmidrule{6-7} \cmidrule{9-10}
    & & {\texttt{ML}} & {\texttt{FTU}} && {\texttt{Adv} ($\lambda = 0.1$)} &{\texttt{Adv} ($\lambda = 0.2$)} && {\texttt{$\text{OB}$}} & {\texttt{$\text{SOB}$}} \\
    \midrule
    \multirow{3}{*}{Crime~\cite{misc_communities_and_crime_183}} 
    & MSE & 0.4751 (0.0336) & 0.4711 (0.0101) && 0.4674 (0.0425) & 0.4511 (0.0503) && \underline{0.4534} (0.0110) & \textbf{0.4491} (0.0155) \\
    & CF-metric & 0.2353 (0.0729) & 0.2467 (0.1633) && \underline{0.0095} (0.0057) & \textbf{0.0066} (0.0029) && 0.1047 (0.0373) & 0.1051 (0.0328) \\
    & Avg. Training Time (s) & 68.9089 (7.4695)  & 62.9821 (6.5349)  &&  136.9746 (4.4359) & 131.8830 (4.4037) && \textbf{69.2535} (5.1429) & 80.7566 (7.4400) \\
    \midrule
    \multirow{3}{*}{Synthetic~\cite{grari2023adversarial}}
    & MSE & 0.0008 (0.0009) & 0.0000 (0.0000) && {0.0658} (0.0731) & 0.1062 (0.2174) &&  \textbf{0.0054} (0.0001) & \textbf{0.0054} (0.0002) \\
    & CF-metric & 0.1414 (0.0008) & 0.1418 (0.0002) && 0.1422 (0.0361) &  \textbf{0.1189} (0.0409) && {0.1309} (0.0023) & \underline{0.1296} (0.0020) \\
    & Avg. Training Time (s) & 36.4220 (13.9463) & 32.9465 (7.0428) && 55.2336 (2.6324) & 63.5752 (0.6989) && \textbf{32.8576} (8.6382) & 37.5425 (10.8821) \\
    \bottomrule
    \end{tabular}
    }
    \label{table:combined}
\end{table*}

In this section, we empirically evaluate the performance of our method using three real-world datasets and two simulated datasets. For the discrete scenario where $Y \in \{0,1\}$ and $B \in \{0,1\}$, we utilize two popular datasets: the Adult income dataset~\cite{misc_census_income_20}, with gender as the sensitive variable (male or female) and whether a person earns over \$50K a year as $Y$, and the COMPAS dataset~\cite{angwin2022machine}, with race as the sensitive variable (Caucasian or non-Caucasian) and two-year recidivism as $Y$. In addition, we explore a synthetic loan decision dataset using similar setup as~\cite{chen2023learning}. For scenarios where both $Y$ and $A$ are continuous, we use the Crime dataset~\cite{misc_communities_and_crime_183}, with the ratio of an ethnic group per population as the sensitive variable and violent crimes per population as $Y$. Additionally, we explore a synthetic scenario involving a pricing algorithm for a fictional car insurance price adapted from~\cite{grari2023adversarial}, following the causal graph in Figure~\ref{fig:illustration_a}. Details of the experimental setup can be found in Appendix~\ref{app:exp}.

\paragraph{Evaluation metrics}
We assess the accuracy of decisions using Mean Squared Error (MSE) for continuous value predictions, and Area Under the Curve (AUC) and Accuracy (ACC) for binary classification problems. In evaluating counterfactual fairness for continuous cases, we employ the CF-Metric from~\cite{grari2023adversarial}, which quantifies the mean change in predictions across counterfactual samples generated via variational inference of the Structural Causal Model (SCM). For classification tasks, we adopt CF-Metrics from~\cite{chen2023learning}, measuring the average shift in predicted scores between sensitive groups, where a value of zero indicates counterfactual fairness. Additionally, we report Affirmative Action (AA) Fairness and Equalized Opportunities (EO) Fairness~\cite{wang2019equal}, which assess fairness in discrete sensitive-variable settings. Detailed metric definitions are provided in Appendix~\ref{app:exp}.
Both AA Fairness and EO Fairness are only defined for scenarios with discrete sensitive variables. Definitions of these metrics are provided in Appendix~\ref{app:exp}.

\paragraph{Results with continuous sensitive variables}

To our knowledge, the approach by~\cite{grari2023adversarial}  (henceforth referred to as \texttt{Adv}) is the only other notable attempt to tackle counterfactual fairness with continuous sensitive variables. In our experiments, we included this method along with a standard machine learning approach (\texttt{ML}) and the Fairness Through Unawareness (\texttt{FTU}) method as baselines. All models were evaluated using a uniform four-layer neural network for prediction tasks.

Figure~\ref{fig:synth_comp} illustrates the fairness-accuracy trade-off, a well-explored topic in the literature~\cite{agarwal2018reductions, baharlouei2019r, Berk2017ACF}. 
Moving along the spectrum, \texttt{Adv} appears with progressively increasing $\lambda$ values, reflecting a shift towards greater emphasis on fairness at the expense of accuracy. Our method is located in the lower-left, demonstrating low MSE and CF metrics, with narrow error bars due to its closed-form solution (Appendix~\ref{app:ob_closed_form}). Furthermore, Table~\ref{table:combined} provides detailed numerical results, including mean and standard deviations from 10 repeated experiments, which consistently demonstrate our method's superior performance in achieving lower MSE and competitive CF-metrics across two datasets with continuous variables. 
Moreover, Table~\ref{table:combined} shows that our method is more time-efficient than \texttt{Adv}.

\begin{table*}[t]
\centering
\caption{Performance comparison on COMPAS data~\cite{angwin2022machine}}
\resizebox{\linewidth}{!}{%
\begin{tabular}{cccccccccccccccc}
\toprule
\multirow{2}{*}{Metrics} & \multicolumn{2}{c}{Baselines} &
\multicolumn{8}{c}{Compared Methods} && \multicolumn{2}{c}{Ours} \\
\cmidrule{2-3} \cmidrule{5-11} \cmidrule{13-14}
{}
& {\texttt{ML}} & {\texttt{FTU}} && {\texttt{FL}} & {\texttt{EO}} & {\texttt{AA}} & {\texttt{$\text{FLAP}_1\texttt{(O)}$}} & {\texttt{$\text{FLAP}_2\texttt{(O)}$}} & {\texttt{$\text{FLAP}_1\texttt{(M)}$}} & {\texttt{$\text{FLAP}_2\texttt{(M)}$}} &&{\texttt{$\text{OB}_1$}} & {\texttt{$\text{OB}_2$}}  \\
\midrule
ACC & 0.5744 & 0.5726 && 0.5598 & \textbf{0.5710} & 0.5609 & 0.5605 & 0.5599 & 0.5607 & 0.5607 && 0.5666 & \underline{0.5674} \\
AUC & 0.7206 & 0.7225 && 0.6928 & \textbf{0.7225} &  0.6927 & 0.6927  & 0.6928 & 0.7015 & \underline{0.7019} && 0.6764 & 0.6744 \\
\midrule
CF-metric & 0.2274 & 0.1406 && 0.0054 & 0.1377 & 0.0060 & 0.0058 & 0.0054 & \textbf{0.0026} & \underline{0.0027} && 0.0060 & 0.0065 \\
EO Fairness & 0.1046 & 0 && 0.1374 & \textbf{0} & 0.1405 & 1.7e-06 & 3.3e-06 & 6.7e-07 & 1.2e-06 && \textbf{0} & \textbf{0} \\
AA Fairness & 0.2258 & 0.1460 && \textbf{0} & 0.1424 & \textbf{0} & 2.9e-07 & 5.6e-07 & 8.2e-07 & 3.0e-07 && \underline{1.6e-16} & \underline{1.1e-16} \\
\bottomrule
\end{tabular}
}
\label{tab:real_compas}
\end{table*}

\begin{table*}[t]
\centering
\caption{Performance comparison 
on Adult data~\cite{misc_adult_2}}

\resizebox{\linewidth}{!}{%
\begin{tabular}{cccccccccccccccc}
\toprule
\multirow{2}{*}{Metrics} & \multicolumn{2}{c}{Baselines} &
\multicolumn{8}{c}{Compared Methods} && \multicolumn{2}{c}{Ours} \\
\cmidrule{2-3} \cmidrule{5-11} \cmidrule{13-14}
{}
& {\texttt{ML}} & {\texttt{FTU}} && {\texttt{FL}} & {\texttt{EO}} & {\texttt{AA}} & {\texttt{$\text{FLAP}_1\texttt{(O)}$}} & {\texttt{$\text{FLAP}_2\texttt{(O)}$}} & {\texttt{$\text{FLAP}_1\texttt{(M)}$}} & {\texttt{$\text{FLAP}_2\texttt{(M)}$}} &&{\texttt{$\text{SOB}_1$}} & {\texttt{$\text{SOB}_2$}}  \\
\midrule
ACC & 0.7612 & 0.7604 && 0.7594 & \textbf{0.7680} & 0.7644 & 0.7357 & 0.7151 & 0.7548 & 0.7594  && \underline{0.7655} & 0.7597 \\
AUC & 0.8128 & 0.8036 && 0.7680 & \textbf{0.7991} & 0.7682 & 0.7682 & 0.7680 & 0.7651 & 0.7649  && \underline{0.7806} & 0.7809 \\
\midrule
CF-metric & 0.2779 & 0.2338 && \textbf{0.0228} & 0.2047 & \underline{0.0268} & 0.0280 & \textbf{0.0228} & 0.0280 & \textbf{0.0228}  && 0.0529 & 0.0600 \\
EO Fairness & 0.1536 & 0 && 0.2853 & \textbf{0} & 0.2811 & 0.2780 & 0.2853 & 0.2780 & 0.2853  && \underline{0.0002} & 0.0005 \\
AA Fairness & 0.3034 & 0.2574 && \textbf{0} & 0.2259 & \textbf{0} & \underline{2.2e-17} & \underline{2.2e-17} & \underline{2.8e-17} & \underline{2.8e-17}  && 0.0001 & 0.0004 \\
\bottomrule
\end{tabular}
}
\label{tab:real_adult}
\end{table*}

\begin{table*}[!t]
\centering
\caption{Performance comparison 
on Synthetic Loan data~\cite{chen2023learning}}
\label{tab:syn_result_loan_2}

\resizebox{\linewidth}{!}{%
\begin{tabular}{cccccccccccccc}
\toprule
\multirow{2}{*}{Metrics} & \multicolumn{2}{c}{Baselines} &
\multicolumn{8}{c}{Compared Methods} && \multicolumn{2}{c}{Ours} \\
\cmidrule{2-3} \cmidrule{5-11} \cmidrule{13-14}
{}
& {\texttt{ML}} & {\texttt{FTU}} && {\texttt{FL}} & {\texttt{EO}} & {\texttt{AA}} & {\texttt{$\text{FLAP}_1\texttt{(O)}$}} & {\texttt{$\text{FLAP}_2\texttt{(O)}$}} & {\texttt{$\text{FLAP}_1\texttt{(M)}$}} & {\texttt{$\text{FLAP}_2\texttt{(M)}$}} &&{\texttt{$\text{OB}_1$}} & {\texttt{$\text{OB}_2$}}  \\
\midrule
ACC & 0.6618 & 0.6481 && 0.6224 & 0.6237 & 0.6224 & 0.6237 & 0.6224 &  0.6237 & 0.6224 && \textbf{0.6406} & \underline{0.6279} \\
AUC & 0.9457 & 0.8986 && \underline{0.5867} & \textbf{0.6682} & 0.5714 & 0.5668 & 0.5837 & 0.5875 & 0.5863 && 0.5704 & 0.5856\\
\midrule
CF-metrics & 0.6291 & 0.3906 && 0.0031 & 0.0355 & 0.0034 & 0.0016 & 0.0032 & \textbf{0.0002} & \textbf{0.0002} && \underline{0.0011} & 0.0026 \\
EO Fairness & 0.5469 & 0 && \underline{0.0156} & \textbf{0} & 0.0336  & 0.0321 & \underline{0.0156} & 0.0301 & 0.0180 && \textbf{0} & \textbf{0} \\
AA Fairness & 0.6235 & 0.4559 && \textbf{5.6e-18} & 0.0370 & \textbf{1.1e-18} & \textbf{3.3e-18} & \textbf{6.7e-18} & 0.0012 &  0.0038 && \textbf{4.6e-17} & \textbf{4.3e-17} \\

\bottomrule
\end{tabular}
}
\end{table*}

\paragraph{Results with discrete sensitive variables}

We compare our method with the following baselines on datasets that include discrete sensitive variables: Machine Learning (\texttt{ML}), a straightforward logistic regression model that uses all variables, regardless of their sensitivity; Fairness through Unawareness (\texttt{FTU}), which employs a logistic model with $\mathbf{A}$, specifically excluding sensitive variables; FairLearning Algorithm (\texttt{FL}), which attains counterfactual fairness by sampling unobserved portions of the causal model using Markov chain Monte Carlo methods~\cite{kusner2017counterfactual}; EO-fair and AA-fair Predictors (\texttt{EO} and \texttt{AA}) by a post-processing algorithm that generate optimal predictors adhering to Equalized Opportunities (EO) and Affirmative Action (AA)~\cite{wang2019equal}; and Fair Learning through Data Pre-processing (\texttt{FLAP}) by~\cite{chen2023learning}, which modifies data prior to model training to enhance fairness. Logistic regression predictors are adopted for all compared models with classification tasks.

Among the compared methods, \texttt{FTU} and \texttt{FLAP} are pre-processing methods, \texttt{FL} is an in-processing approach, and \texttt{AA} and \texttt{EO} are post-processing approaches.
For discrete sensitive variables, both \texttt{FLAP} and our proposed method (\texttt{OB}) can include sensitive variables in the predictor class. Specifically, aside from the predictor class ${f_{\widehat{Y}}(\mathbf{A}}): \mathcal{A \rightarrow Y}$ discussed in Section \ref{subsec:debiasing}, an standard machine learning predictor  ${f_{\widehat{Y}}(\mathbf{A}, \mathbf{B}): \mathcal{A \times B} \rightarrow \mathcal{Y}}$ that utilizes both sensitive and non-sensitive variables, and an Averaged Machine Learning (AML) predictor $f_{\widehat{Y}}^{'}(\mathbf{A}) = \sum f_{\widehat{Y}}(\mathbf{A}, \mathbf{B})\mathbb{P}(\mathbf{B})$ can be constructed respectively.
We denote scenarios involving predictor $f^{'}_{\widehat{Y}}(\mathbf{A})$ or $f_{\widehat{Y}}(\mathbf{A})$ as
$\texttt{OB}_1$ and $\texttt{OB}_2$, respectively. 
Similarly, we use $\texttt{FLAP}_1$ and $\texttt{FLAP}_2$ for the respective predictor scenarios. 
Additionally,~\cite{chen2023learning} introduces two pre-processing methods, Orthogonalization and Marginal Distribution Mapping, denoted as \texttt{FLAP(O)} and \texttt{FLAP(M)}. 



As shown in Table~\ref{tab:real_compas}, \ref{tab:real_adult}, and \ref{tab:syn_result_loan_2}, the accuracy of \texttt{OB} is consistently among the highest, often ranking first or second. 
This supports our claim that the proposed \texttt{OB} algorithm effectively preserves most of the information in the non-sensitive data. Additionally, its CF-metric is comparable to those of \texttt{FLAP} and \texttt{FL},
which are also designed for counterfactual fairness with discrete sensitive variables. This suggests that \texttt{OB} also achieves reasonable counterfactual fairness in discrete cases.
Furthermore, despite the general incompatibility between counterfactual fairness and popular group fairness metrics like Equalized Opportunities, EO can also be encouraged with \texttt{OB}. Empirically, \texttt{OB} achieves relatively low EO and AA Fairness metrics compared to \texttt{FLAP} and \texttt{FL}.
More details and proof can be found in Appendix~\ref{app:eo_fair}.

In summary, across all discrete datasets, our approach consistently demonstrates its effectiveness in achieving a better overall balance between accuracy and counterfactual fairness, even when the data do not fully meet our theoretical assumptions.

\paragraph{Sensitivity Analysis}

We investigate the sensitivity of \texttt{OB} to its key hyperparameter, the rank $k$, which governs the trade-off between fairness and information retention through low-rank approximation. A lower $k$ enforces stronger fairness by removing more bias-aligned variance but may degrade predictive accuracy by discarding informative structure in the data. For \texttt{SOB}, although it introduces additional parameters such as the convergence threshold $\eta$ in Algorithm~\ref{alg:og}, we find these have negligible effect on overall performance and thus focus primarily on $k$.

We first analyze the COMPAS dataset, which contains five features and a single binary sensitive attribute. This low-dimensional setting allows us to vary the rank $k \in \{3, 4, 5\}$. As shown in Table~\ref{table:sensitivity}, we observe a consistent improvement in both predictive performance (accuracy) and fairness outcomes (CF-metric and AA Fairness) as $k$ increases. This behavior aligns with prior findings in Principal Component Analysis (PCA)~\cite{aliverti2021removing}, where greater rank allows for better reconstruction.

To test robustness in a high-dimensional setting, we perform a similar sensitivity analysis on the Crime dataset, which includes 121 features. We vary the rank $k$ from 90 to 120 and report the test mean squared error (MSE) and CF-metric in Table~\ref{tab:k_sensitivity}. As expected, lower values of $k$ incur a slight increase in prediction error due to more aggressive dimensionality reduction. However, the observed differences in both MSE and counterfactual fairness are modest across the range, indicating that \texttt{OB} remains stable and reliable under different rank configurations.

Together, these results highlight the robustness of \texttt{OB} to its rank parameter. While careful tuning can yield marginal improvements, the method performs consistently well even without hyperparameter tuning in our experiments.

\begin{table}
\centering
\caption{Sensitivity Analysis of \texttt{OB}$_1$ and \texttt{OB}$_2$ to Rank Parameter $k$ on the COMPAS Dataset
}
\resizebox{0.7\linewidth}{!}{%
\begin{tabular}{cccccccc}
\toprule
\multirow{2}{*}{Metrics} &
\multicolumn{3}{c}{\texttt{$\text{OB}_1$}} && \multicolumn{3}{c}{\texttt{$\text{OB}_2$}} \\
\cmidrule{2-4} \cmidrule{6-8}
& {$k$ = 3} & {$k$ = 4} & {$k$ = 5} && {$k$ = 3} & {$k$ = 4} & {$k$ = 5} \\
\midrule
ACC & 0.5409 & 0.5507 & 0.5666 && 0.5420 & 0.5516 & 0.5674 \\
CF-metric & 0.0120 & 0.0157 & 0.0060 && 0.0129 & 0.0162 & 0.0065 \\
EO Fairness & 0 & 0 & 0 && 0 & 0 & 0 \\
AA Fairness & 0.0006 & 0.0003 & 1.6e-16 && 0.0006 & 0.0003 & 1.1e-16
 \\
\bottomrule
\end{tabular}
}
\label{table:sensitivity}
\end{table}

\begin{table}
\centering
\caption{Sensitivity Analysis of \texttt{$\text{OB}_1$} to Rank Parameter $k$ on the Crime Dataset}
\label{tab:k_sensitivity}
\resizebox{0.8\linewidth}{!}{%
\begin{tabular}{cccccc}
\toprule
Metrics & $k = 90$ & $k = 100$ & $k = 110$ & $k = 120$ \\
\midrule
Test MSE & 0.4197 $\pm$ 0.0136 & 0.4184 $\pm$ 0.0119 & \textbf{0.4156 $\pm$ 0.0127} & 0.4158 $\pm$ 0.0134 \\
CF Metric & 0.0332 $\pm$ 0.0099 & 0.0321 $\pm$ 0.0099 & \textbf{0.0320 $\pm$ 0.0107} & 0.0327 $\pm$ 0.0118 \\
\bottomrule
\end{tabular}
}
\end{table}



\section{Conclusion}
\label{sec:conclusion}

In conclusion, this paper demonstrates that counterfactual fairness can be achieved by ensuring the establishing orthogonality of non-sensitive and sensitive variables under certain conditions. Based on this insight, we introduce the Orthogonal to Bias (\texttt{OB}) algorithm, a novel data pre-processing method that removes bias while preserving as much information as possible. \texttt{OB} is model-agnostic, making it adaptable across various machine learning models. Empirical evaluations on both simulated and real-world datasets show that \texttt{OB} effectively balances fairness and accuracy, outperforming existing methods.
Despite its advantages, \texttt{OB} has limitations. As an offline pre-processing approach, it may not integrate as seamlessly with real-time systems as in-processing methods. Additionally, the theoretical guarantees of \texttt{OB} depend on joint normality which may not hold in some practical scenarios. 

\bibliographystyle{unsrt}
\bibliography{ref}

\begin{thebibliography}{10}

\bibitem{bolukbasi2016man}
Tolga Bolukbasi, Kai-Wei Chang, James~Y Zou, Venkatesh Saligrama, and Adam~T Kalai.
\newblock Man is to computer programmer as woman is to homemaker? debiasing word embeddings.
\newblock {\em Advances in neural information processing systems}, 29, 2016.

\bibitem{hardt2016equality}
Moritz Hardt, Eric Price, and Nathan Srebro.
\newblock Equality of opportunity in supervised learning.
\newblock In {\em Proceedings of the 30th International Conference on Neural Information Processing Systems}, NIPS'16, page 3323–3331, Red Hook, NY, USA, 2016. Curran Associates Inc.

\bibitem{dwork2012fairness}
Cynthia Dwork, Moritz Hardt, Toniann Pitassi, Omer Reingold, and Richard Zemel.
\newblock Fairness through awareness.
\newblock In {\em Proceedings of the 3rd Innovations in Theoretical Computer Science Conference}, ITCS '12, page 214–226, New York, NY, USA, 2012. Association for Computing Machinery.

\bibitem{grgic2016case}
Nina Grgic-Hlaca, Muhammad~Bilal Zafar, Krishna~P Gummadi, and Adrian Weller.
\newblock The case for process fairness in learning: Feature selection for fair decision making.
\newblock In {\em NIPS symposium on machine learning and the law}, volume~1, page~11, Barcelona, Spain, 2016. Barcelona, Spain, Curran Associates, Inc.

\bibitem{zemel2013learning}
Rich Zemel, Yu~Wu, Kevin Swersky, Toni Pitassi, and Cynthia Dwork.
\newblock Learning fair representations.
\newblock In Sanjoy Dasgupta and David McAllester, editors, {\em Proceedings of the 30th International Conference on Machine Learning}, volume~28 of {\em Proceedings of Machine Learning Research}, pages 325--333, Atlanta, Georgia, USA, 17--19 Jun 2013. PMLR.

\bibitem{zafar2017fairness}
Muhammad~Bilal Zafar, Isabel Valera, Manuel~Gomez Rogriguez, and Krishna~P. Gummadi.
\newblock Fairness constraints: Mechanisms for fair classification.
\newblock In Aarti Singh and Jerry Zhu, editors, {\em Proceedings of the 20th International Conference on Artificial Intelligence and Statistics}, volume~54 of {\em Proceedings of Machine Learning Research}, pages 962--970, Sydney, Australia, 20--22 Apr 2017. PMLR.

\bibitem{kusner2017counterfactual}
Matt~J Kusner, Joshua Loftus, Chris Russell, and Ricardo Silva.
\newblock Counterfactual fairness.
\newblock In I.~Guyon, U.~Von Luxburg, S.~Bengio, H.~Wallach, R.~Fergus, S.~Vishwanathan, and R.~Garnett, editors, {\em Advances in Neural Information Processing Systems}, volume~30, Long Beach, CA, USA, 2017. Curran Associates, Inc.

\bibitem{10.1145/3597199}
Zeyu Tang, Jiji Zhang, and Kun Zhang.
\newblock What-is and how-to for fairness in machine learning: A survey, reflection, and perspective.
\newblock {\em ACM Comput. Surv.}, 55(13s), jul 2023.

\bibitem{wang2019equal}
Yixin Wang, Dhanya Sridhar, and David~M Blei.
\newblock Equal opportunity and affirmative action via counterfactual predictions.
\newblock {\em arXiv preprint arXiv:1905.10870}, 2019.

\bibitem{chen2023learning}
Rui~Song Haoyu~Chen, Wenbin~Lu and Pulak Ghosh.
\newblock On learning and testing of counterfactual fairness through data preprocessing.
\newblock {\em Journal of the American Statistical Association}, 0(0):1--11, 2023.

\bibitem{chiappa2019path}
Silvia Chiappa.
\newblock Path-specific counterfactual fairness.
\newblock volume~33, pages 7801--7808, Jul. 2019.

\bibitem{grari2023adversarial}
Vincent Grari, Sylvain Lamprier, and Marcin Detyniecki.
\newblock Adversarial learning for counterfactual fairness.
\newblock {\em Machine Learning}, 112(3):741--763, 2023.

\bibitem{barocas2016big}
Solon Barocas and Andrew~D Selbst.
\newblock Big data's disparate impact.
\newblock {\em Calif. L. Rev.}, 104:671, 2016.

\bibitem{pearl2009causality}
Judea Pearl.
\newblock {\em Causality: Models, Reasoning and Inference}.
\newblock Cambridge University Press, USA, 2nd edition, 2009.

\bibitem{kamiran2012data}
Faisal Kamiran and Toon Calders.
\newblock Data preprocessing techniques for classification without discrimination.
\newblock {\em Knowledge and information systems}, 33(1):1--33, 2012.

\bibitem{luong2011k}
Binh~Thanh Luong, Salvatore Ruggieri, and Franco Turini.
\newblock k-nn as an implementation of situation testing for discrimination discovery and prevention.
\newblock In {\em Proceedings of the 17th ACM SIGKDD International Conference on Knowledge Discovery and Data Mining}, KDD '11, page 502–510, New York, NY, USA, 2011. Association for Computing Machinery.

\bibitem{wei2024assessing}
Song Wei, Xiangrui Kong, Alinson Santos~Xavier, Shixiang Zhu, Yao Xie, and Feng Qiu.
\newblock Assessing electricity service unfairness with transfer counterfactual learning.
\newblock {\em arXiv preprint arXiv:2310.03258}, 2024.

\bibitem{khademi2019fairness}
Aria Khademi, Sanghack Lee, David Foley, and Vasant Honavar.
\newblock Fairness in algorithmic decision making: An excursion through the lens of causality.
\newblock In {\em The World Wide Web Conference}, pages 2907--2914, New York; NY; United States, 2019. Association for Computing Machinery.

\bibitem{mishler2021fairness}
Alan Mishler, Edward~H. Kennedy, and Alexandra Chouldechova.
\newblock Fairness in risk assessment instruments: Post-processing to achieve counterfactual equalized odds.
\newblock In {\em Proceedings of the 2021 ACM Conference on Fairness, Accountability, and Transparency}, FAccT '21, page 386–400, New York, NY, USA, 2021. Association for Computing Machinery.

\bibitem{agarwal2018reductions}
Alekh Agarwal, Alina Beygelzimer, Miroslav Dudik, John Langford, and Hanna Wallach.
\newblock A reductions approach to fair classification.
\newblock In Jennifer Dy and Andreas Krause, editors, {\em Proceedings of the 35th International Conference on Machine Learning}, volume~80 of {\em Proceedings of Machine Learning Research}, pages 60--69. PMLR, 10--15 Jul 2018.

\bibitem{baharlouei2019r}
Sina Baharlouei, Maher Nouiehed, Ahmad Beirami, and Meisam Razaviyayn.
\newblock R$\backslash$'enyi fair inference.
\newblock {\em arXiv preprint arXiv:1906.12005}, 2019.

\bibitem{Berk2017ACF}
Richard~A. Berk, Hoda Heidari, Shahin Jabbari, Matthew Joseph, Michael Kearns, Jamie Morgenstern, Seth Neel, and Aaron Roth.
\newblock A convex framework for fair regression.
\newblock {\em ArXiv}, abs/1706.02409, 2017.

\bibitem{aliverti2021removing}
Emanuele Aliverti, Kristian Lum, James~E Johndrow, and David~B Dunson.
\newblock Removing the influence of group variables in high-dimensional predictive modelling.
\newblock {\em Journal of the Royal Statistical Society. Series A,(Statistics in Society)}, 184(3):791, 2021.

\bibitem{james1954normal}
Alan~T James.
\newblock Normal multivariate analysis and the orthogonal group.
\newblock {\em The Annals of Mathematical Statistics}, 25(1):40--75, 1954.

\bibitem{zou2006sparse}
Hui Zou, Trevor Hastie, and Robert Tibshirani.
\newblock Sparse principal component analysis.
\newblock {\em Journal of computational and graphical statistics}, 15(2):265--286, 2006.

\bibitem{witten2009penalized}
Daniela~M Witten, Robert Tibshirani, and Trevor Hastie.
\newblock A penalized matrix decomposition, with applications to sparse principal components and canonical correlation analysis.
\newblock {\em Biostatistics}, 10(3):515--534, 2009.

\bibitem{misc_communities_and_crime_183}
Michael Redmond.
\newblock {Communities and Crime}.
\newblock UCI Machine Learning Repository, 2009.
\newblock {DOI}: https://doi.org/10.24432/C53W3X.

\bibitem{misc_census_income_20}
Ron Kohavi.
\newblock {Census Income}.
\newblock UCI Machine Learning Repository, 1996.
\newblock {DOI}: https://doi.org/10.24432/C5GP7S.

\bibitem{angwin2022machine}
Julia Angwin, Jeff Larson, Surya Mattu, and Lauren Kirchner.
\newblock Machine bias.
\newblock In {\em Ethics of data and analytics}, pages 254--264. Auerbach Publications, 2022.

\bibitem{misc_adult_2}
Barry Becker and Ronny Kohavi.
\newblock {Adult}.
\newblock UCI Machine Learning Repository, 1996.
\newblock {DOI}: https://doi.org/10.24432/C5XW20.

\bibitem{bishop2006pattern}
Christopher~M. Bishop.
\newblock {\em Pattern Recognition and Machine Learning (Information Science and Statistics)}.
\newblock Springer-Verlag, Berlin, Heidelberg, 2006.

\bibitem{hastie2015statistical}
Carl~M. O'Brien.
\newblock Statistical learning with sparsity: The lasso and generalizations.
\newblock {\em International Statistical Review}, 84(1):156--157, 2016.

\end{thebibliography}

\clearpage
\appendix

\section{Counterfactual Fairness Guarantee under Joint Normality}
\label{app:joint-normal-proof}

We present an extension to Theorem~\ref{thm:elliptical}, showing that counterfactual fairness holds when sensitive and non-sensitive variables are jointly normal and uncorrelated for general predictors $\widehat{Y}$.

\begin{assumption}[]
\label{as:assumption_1}
Given the structural model defined in \eqref{eq:scm}, the sensitive variable $A$ and non-sensitive variables $B$ are joint normal.
\end{assumption}
Building on this assumption, we present the following theorem:
\begin{proposition}[]
\label{thm_1}
Under Assumption~\ref{as:assumption_1}, $\widehat{Y}$ is counterfactually fair when $A$ and $B$ are uncorrelated.
\end{proposition}
\begin{proof} 
We use Pearl's ``Abduction–Action–Prediction'' framework~\cite{pearl2009causality}. For binary $\widehat{Y}$, the counterfactual expectation given observed $(A = \mathbf{a}^*, B = \mathbf{b}^*)$ is:
\begin{equation}
\begin{aligned}
    \mathbb{E}(\widehat{Y}_{\mathbf{b}'}|B=\mathbf{b}^*, A=\mathbf{a}^*)
    = \int f_{\widehat{Y}}\left(f_A\left(\mathbf{b}', u\right); \mathcal{D}\right)
    \mathbb{P}_{\varepsilon_A \mid B, A}\left(u \mid B=\mathbf{b}^*, A=\mathbf{a}^*\right) \mathrm{d}u ,
    \label{eq:theorom_1}
\end{aligned}
\end{equation}
where 
$f_{\widehat{Y}}(\cdot;\mathcal{D}): \mathcal{A} \rightarrow \mathcal{Y}$ denotes the predictor of $\widehat{Y}$ trained using data $\mathcal{D}$ and
$\mathbb{P}_{\varepsilon_A \mid B, A}\left(u \mid B=\mathbf{b}^*, A=\mathbf{a}^*\right)$ denotes the conditional density of $\varepsilon_A$ given $B = \mathbf{b}^*$ and $A = \mathbf{a}^*$.
To argue for counterfactual fairness, it suffices to show 
\[
    \mathbb{E}(\widehat{Y}_{\mathbf{b}'}|B=\mathbf{b}^*, A=\mathbf{a}^*) = \mathbb{E}(\widehat{Y}_{\mathbf{b}^*}|B=\mathbf{b}^*, A=\mathbf{a}^*),
\]
if the data generating process for the observed data $f_A\left(\mathbf{b}, u\right)$ does not depend on the value of $\mathbf{b}$, indicating $A$'s independence from $B$.





Therefore, following the classical result that jointly normal random variables are independent if and only if they are uncorrelated, uncorrelatedness under the jointly normal SCM in \eqref{eq:scm} implies independence. Consequently, $\widehat{Y}$ based on $A$ is counterfactually fair.

\end{proof}

\section{Analyzing the Elliptical Distribution of Linear Combinations in SCM}
\label{app:proof}

This section demonstrates how the Structural Causal Model (SCM) depicted in Figure~\ref{fig:illustration_a} aligns with Assumption~\ref{as:elliptical} under mild conditions. For simplicity, we assume both $A$ and $B$ are one-dimensional variables.

Consider the exogenous variables $\varepsilon_A$ and $\varepsilon_B$ to be independent and drawn from a joint elliptical distribution:
\[
(\varepsilon_A, \varepsilon_B) \sim \mathcal{E}_2(\boldsymbol{\mu}, \Sigma, \psi),
\]
where $\boldsymbol{\mu} = [\mu_A, \mu_B]^\top$, $\Sigma_\varepsilon = \text{diag}(\sigma_A^2, \sigma_B^2)$, and $\psi$ is a characteristic generator. The independence between $\varepsilon_A$ and $\varepsilon_B$ implies that $\Sigma$ is diagonal.

Define $A$ and $B$ as linear functions of the exogenous variables:
\begin{align*}
A &= a\varepsilon_A + b\varepsilon_B + c, \\
B &= d\varepsilon_B + e,
\end{align*}
where $a$, $b$, $c$, $d$, and $e$ are constants.

Since elliptical distributions are closed under affine transformations, the pair $(A, B)$ follows a bivariate elliptical distribution:
\[
(A, B) \sim \mathcal{E}_2(\boldsymbol{\mu}, \Sigma, \psi),
\]
where $\boldsymbol{\mu} = \mathbb{E}[A, B]^\top$ and $\Sigma$ is a positive-definite scatter matrix derived from the affine transformation of $(\varepsilon_A, \varepsilon_B)$.

Thus, under this formulation, $(A, B)$ satisfy Assumption~\ref{as:elliptical}. The synthetic data used for continuous sensitive variables, described in Appendix~\ref{app:synthetic}, is generated based on a jointly normal case of this structure and conforms to the elliptical distribution assumption.

\section{Closed-form \texttt{OB} Solution Derivation}
\label{app:ob_closed_form}
We start by considering \eqref{eqn:OG_2} when $k$ = 1.
In the main text, we write the reconstruction as $\widetilde{\mathbf{A}} = \mathbf{S} \mathbf{U}^\top$ with $\mathbf{S} \in \mathbb{R}^{n \times k}$ and $\mathbf{U} \in \mathbb{R}^{q \times k}$. For $k = 1$, we simplify this to $\mathbf{s} \in \mathbb{R}^n$ and $\mathbf{u} \in \mathbb{R}^q$, so that $\widetilde{\mathbf{A}} = \mathbf{s} \mathbf{u}^\top$.
In this case, the \texttt{OB} algorithm aims to find the closest rank-1 matrix approximation to the original set of data that satisfies the orthogonal condition. 
To handle the orthogonality constraint $\sum_{i=1}^{n} s_i b_i = 0$ for $k$ = 1, we introduce a Lagrange multiplier $\lambda_1 \in \mathbb{R}$ and write the Lagrangian:
\begin{equation}
\mathcal{L}(\mathbf{s}, \mathbf{u}, \lambda_1)
=
\frac{1}{n} \sum_{i=1}^{n} \left\| \mathbf{a}_i - s_i \mathbf{u}^{\top} \right\|_2^2
+ \frac{2\lambda_1}{n} \sum_{i=1}^{n} s_i b_i.
\label{eqn:A_1} 
\end{equation}



Expanding the squared norm and using $\|\mathbf u\|_2=1$ gives
\[
\frac{\partial\mathcal L}{\partial s_i}
    \;=\;
    \frac{2}{n}
        \bigl(-\mathbf a_i^{\top}\mathbf u+s_i+\lambda_1 b_i\bigr).
\]

The function is quadratic, and its partial derivative with respect to $s_{i}$ is
\begin{equation}
s_i
    \;=\;
    \mathbf a_i^{\top}\mathbf u-\lambda_1 b_i.
\label{append:s_sol}
\end{equation}
 So the optimal score for the $i$-th subject is obtained by projecting the observed data onto the first basis and then subtracting $\lambda_1 b$.  The constraint does not involve the orthonormal basis $u_1$, hence the solution of \eqref{eqn:A_1} for $u_1$ is equivalent to the unconstrained scenario. A standard result of linear algebra states that the optimal $u_1$ for \eqref{eqn:A_1} without constraints equivalent to the first right singular vector of $\mathbf{A}$, or equivalently to the first eigenvector of the matrix $\mathbf{A}^\top\mathbf{A}$~\cite{witten2009penalized}. Plugging in the solution for $s_i$ into the constraint $\sum_{i=1}^n s_i b_i = 0$ and solving for $\lambda_1$ gives:
\begin{equation}
\sum_{i=1}^{n} (\mathbf{a}_i^{\top} \mathbf{u} - \lambda_1 b_i) b_i = 0.
\end{equation}
Therefore,
\begin{equation}
\lambda_1 = \frac{\sum_{i=1}^{n} (\mathbf{a}_i^{\top} \mathbf{u}) b_i}{\sum_{i=1}^{n} b_i^2}
= \frac{\langle \mathbf{A} \mathbf{u}, \mathbf{b} \rangle}{\langle \mathbf{b}, \mathbf{b} \rangle},
\label{eqn:A_2}
\end{equation}
which shows that $\lambda_1$ is a least squares estimate of the projection of $\mathbf{A} \mathbf{u}$ onto $\mathbf{b}$.

Now consider the more general case when $k > 1$. Let $\mathbf{U} = [\mathbf{u}_1, \dots, \mathbf{u}_k] \in \mathbb{R}^{q \times k}$ denote the orthonormal basis vectors such that $\mathbf{U}^\top \mathbf{U} = \mathbf{I}_k$, and let $\mathbf{S} = [s_{ij}] \in \mathbb{R}^{n \times k}$ collect the score coefficients. The derivatives with respect to the generic element $s_{ij}$ can be calculated similarly due to the orthonormality of $\mathbf{U}$, which simplifies the computation. The optimal score for sample $i$ and component $j$ is given by:
\begin{equation}
s_{ij} = \mathbf{a}_i^{\top} \mathbf{u}_j - \lambda_j b_i,
\label{eqn:A_3}
\end{equation}
since $\mathbf{u}_j^\top \mathbf{u}_l = 0$ for $j \ne l$ and $\mathbf{u}_j^\top \mathbf{u}_j = 1$ for $j = 1, \dots, k$.

The global solution for $\boldsymbol{\lambda} = (\lambda_1, \dots, \lambda_k)$ can be derived from least squares projection, since we can interpret \eqref{eqn:A_3} as a multivariate linear regression problem where the $k$ columns of the projected matrix $\mathbf{A} \mathbf{U}$ are response variables and $\mathbf{b}$ is the covariate. Therefore, the optimal value for each $\lambda_j$ is given by:
\begin{equation}
\lambda_j = \frac{\langle \mathbf{A} \mathbf{u}_j, \mathbf{b} \rangle}{\langle \mathbf{b}, \mathbf{b} \rangle}
= \frac{\mathbf{b}^\top \mathbf{A} \mathbf{u}_j}{\|\mathbf{b}\|_2^2}.
\end{equation}

This results in the closed-form solution in \eqref{eqn:og_solution}. For a more complete proof and discussion of the implications of the solution, we refer to~\cite{aliverti2021removing}. For example, as noted by~\cite{aliverti2021removing},  an intuitive interpretation of the solution in \eqref{eqn:A_3} is that the optimal scores for the $j$-th dimension are obtained by projecting the original data over the $j$-th basis and then subtracting $j$-times the observed value of $b$. Moreover, as the constraints of \texttt{OB} do not involve any vector $u_j$, the optimization with respect to the basis can be derived from known results in linear algebra. The optimal value for the vector $u_j$, with $j = 1,....,k$, is equal to the first k right singular values of $A$, sorted accordingly to the associated singular values~\cite{bishop2006pattern,hastie2015statistical}.

\section{Formulation of \texttt{SOB}}
\label{append:sob}

To extend the applicability of the \texttt{OB} algorithm to high-dimensional settings, we introduce an $\ell_1$-norm penalty on the matrix $\mathbf{U}$ to promote sparsity and enhance numerical stability. The resulting variant, \texttt{SOB} (Sparse Orthogonal to Bias), is formulated as:
\begin{equation}
\begin{aligned}
\min_{\mathbf{S},\, \mathbf{U}}~ & \left\| \mathbf{A} - \mathbf{S} \mathbf{U}^\top \right\|_F^2 \\
\text{s.t.}~ 
& \|\mathbf{u}_j\|_2 \leq 1,\ \|\mathbf{u}_j\|_1 \leq t,\ \|\mathbf{s}_j\|_2 \leq 1, \\
& \mathbf{s}_j^\top \mathbf{s}_\ell = 0 \text{ for } \ell \ne j,\quad
\mathbf{s}_j^\top \mathbf{B} = 0,\quad j = 1,\dots,k.
\end{aligned}
\end{equation}

Here, $\mathbf{A} \in \mathbb{R}^{n \times q}$ and $\mathbf{B} \in \mathbb{R}^{n \times p}$ represent the non-sensitive and sensitive data, respectively. We solve the non-convex problem via alternating minimization~\cite{aliverti2021removing}, optimizing $\mathbf{U}$ and $\mathbf{S}$ iteratively: fixing $\mathbf{S}$ leads to a sparse matrix decomposition problem, while fixing $\mathbf{U}$ yields a constrained quadratic program. The procedure iterates until convergence and ensures that the learned representations are both orthogonal and decorrelated from sensitive attributes. See Algorithm~\ref{alg:og} for implementation details.

\section{Experiment Details}
\label{app:exp}
Experiments were run in Jupyter on a 16 GB RAM machine with a 16 GB T4 GPU. For continuous tasks (Crime and Synthetic), results were averaged over 10 runs. Regression models used a four-layer NN with MSE loss, trained for 2000 epochs (batch size 256). Input size was $q+1$ for \texttt{Adv} and \texttt{ML}, and $q$ for \texttt{FTU}, \texttt{OB}, and \texttt{SOB}. Classification tasks used logistic regression. Rank $k$ was set to 4 (Crime) and 20 (Synthetic).

\subsection{Evaluation Metrics}
\paragraph{CF-Metric (Continuous)}  
The counterfactual fairness (CF) metric for continuous sensitive variable settings quantifies the change in predicted outcomes under hypothetical changes to the sensitive attribute $B$, holding other variables constant~\cite{chen2023learning}:
\begin{equation}
\textstyle
CF = \frac{1}{m_{\text{test}}}\sum_{i=1}^{m_{\text{test}}} 
\mathbb{E}_{(\mathbf{A}'_i,\mathbf{b}'_i) \sim C(i)} \left[ 
\left| \Delta(f_{\widehat Y}(\theta(\mathbf{A}_i, \mathbf{b}_i)), f_{\widehat Y}(\theta(\mathbf{A}'_i, \mathbf{b}'_i))) \right| 
\right],
\end{equation}
where $C(i)$ denotes the counterfactual samples for test point $i$, $\theta$ is the transformation model (e.g., a VAE for \texttt{Adv}, or the \texttt{OB} projection for our method), and $\Delta$ is a prediction discrepancy function.

\paragraph{CF-Metric (Discrete)}  
For discrete sensitive variables, CF is defined as~\cite{chen2023learning}:
\begin{equation}
\begin{aligned}
CF = \max_{r,h \in \mathcal{B}} \left( \frac{1}{m_{\text{test}}} \sum_{i=1}^{m_{\text{test}}} \left| f_{\widehat Y}(\theta(r, \widehat{a}_M^D(r, b_i, a_i))) - f_{\widehat Y}(\theta(b, \widehat{a}_M^D(h, b_i, a_i))) \right| \right),
\end{aligned}
\end{equation}
where $\widehat{a}_M^D(h, b_i, a_i)$ denotes the mapped non-sensitive features under group $h$. A lower CF value indicates higher counterfactual fairness.

Definitions of AA and EO fairness metrics follow those in~\cite{wang2019equal}.

\subsection{Real-world Datasets}

\paragraph{Adult} The Adult Income dataset~\cite{misc_adult_2} predicts whether income exceeds \$50K based on features like sex, race, age, and occupation, with sex and race as sensitive variables. It includes 32,561 training and 16,281 test samples. Given the dataset size, we apply \texttt{SOB}. As shown in Table~\ref{tab:real_adult}, \texttt{SOB} achieves the highest accuracy—even surpassing the vanilla \texttt{ML} model that uses sensitive attributes—and performs competitively on fairness metrics, benefiting from regularization as noted in~\cite{aliverti2021removing, bishop2006pattern}.

\paragraph{COMPAS} The COMPAS dataset~\cite{angwin2022machine} includes demographic and criminal history data for 10,000 defendants. The task is to predict recidivism within two years. Sensitive attributes include race and sex.

\paragraph{Crime} The Crime dataset~\cite{misc_communities_and_crime_183} contains 121 community features. The sensitive variable is the proportion of a racial group, and the target is the violent crime rate per population.

\subsection{Simulated Datasets}
\label{app:synthetic}

\subsubsection{Synthetic Insurance Price Dataset}

    

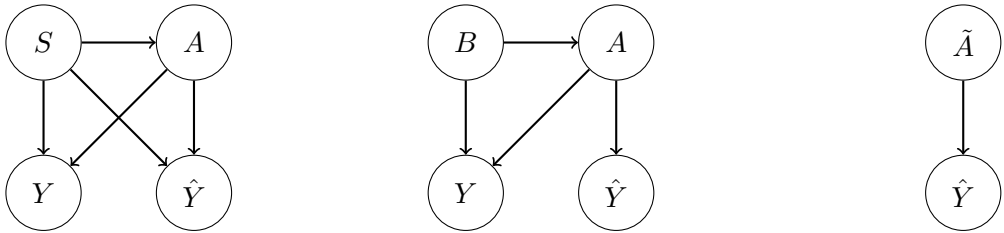
\begin{figure}[ht]
\centering
\begin{subfigure}{.32\textwidth}
    \centering
    \begin{tikzpicture}[node distance=2cm and 2cm]
        \node (B) [circle, draw, minimum size=1cm] at (0,0) {$S$};
        \node (A) [circle, draw, minimum size=1cm] at (2,0) {$A$};
        \node (Yhat) [circle, draw, minimum size=1cm] at (2,-2) {$\widehat{Y}$};
        \node (Y) [circle, draw, minimum size=1cm] at (0,-2) {$Y$};
        
        \draw[->,thick] (B) -- (A);
        \draw[->,thick] (B) -- (Yhat);
        \draw[->,thick] (A) -- (Y);
        \draw[->,thick] (B) -- (Y);
        \draw[->,thick] (A) -- (Yhat);
    \end{tikzpicture}
    \caption{The SCM with typical predictor}
    \label{fig:sub1}
\end{subfigure}%
\hfill
\begin{subfigure}{.32\textwidth}
    \centering
    \begin{tikzpicture}[node distance=2cm and 2cm]
        \node (B) [circle, draw, minimum size=1cm] at (0,0) {$B$};
        \node (A) [circle, draw, minimum size=1cm] at (2,0) {$A$};
        \node (Yhat) [circle, draw, minimum size=1cm] at (2,-2) {$\widehat{Y}$};
        \node (Y) [circle, draw, minimum size=1cm] at (0,-2) {$Y$};
        
        \draw[->,thick] (B) -- (A);
        \draw[->,thick] (A) -- (Yhat);
        \draw[->,thick] (A) -- (Y);
        \draw[->,thick] (B) -- (Y);
    \end{tikzpicture}
    \caption{The SCM with FTU predictor}
    \label{fig:sub2}
\end{subfigure}%
\hfill
\begin{subfigure}{.32\textwidth}
    \centering
    \begin{tikzpicture}[node distance=2cm and 2cm]
        \node (At) [circle, draw, minimum size=1cm] at (1,0) {$\tilde{A}$};
        \node (Yhat) [circle, draw, minimum size=1cm] at (1,-2) {$\widehat{Y}$};

        \draw[->,thick] (At) -- (Yhat);
    \end{tikzpicture}
    \caption{The SCM with \texttt{OB} and \texttt{SOB}}
    \label{fig:sub3}
\end{subfigure}
\caption{Causal diagrams showing the relationships between $S$, $A$, $Y$, and $\widehat{Y}$.}
\end{figure}

We simulate continuous sensitive variable $B$ and decision variable $Y$ following Figure~\ref{fig:illustration_a}~\cite{grari2023adversarial}:

\[
U \sim \mathcal{N}\left(
\begin{bmatrix}
0 \\
0.5 \\
1 \\
1.5 \\
2 \\
\end{bmatrix},
\begin{bmatrix}
1 & 0 & 0 & 0 & 0 \\
0 & 4 & 0 & 0 & 0 \\
0 & 0 & 2 & 0 & 0 \\
0 & 0 & 0 & 3 & 0 \\
0 & 0 & 0 & 0 & 2 \\
\end{bmatrix}
\right)
\]

The non-sensitive variables $A_1, A_2, A_3, A_4$ are defined as follows:
\begin{align*}
A_1 &\sim \mathcal{N}\left(7 + 0.1 B + \varepsilon_1 + \varepsilon_2 + \varepsilon_3,\; 1\right), \\
A_2 &\sim \mathcal{N}\left(80 + B + \varepsilon_2,\; 10\right), \\
A_3 &\sim \mathcal{N}\left(200 + 5 B + 5 \varepsilon_3,\; 20\right), \\
A_4 &\sim \mathcal{N}\left(10^4 + 5 B + \varepsilon_4 + \varepsilon_5,\; 1000\right).
\end{align*}
where $B$ is defined by the vector:
\[
B \sim \mathcal{N}(\begin{bmatrix} 45, 5 \end{bmatrix}).
\]

The response variable $Y$ is modeled as:
\[
Y \sim \mathcal{N}(2 \cdot (7 \cdot B + 20 \cdot \sum_{j} A_j), 0.1).
\]

\subsubsection{Synthetic Loan Decision Dataset}
\label{sec:synthetic_app}

We also include a synthetic loan dataset adapted from~\cite{chen2023learning} with binary outcomes. This setup allows controlled experimentation by varying levels of group-based disparities and averaging results over repeated trials.

Applicants belong to three race groups $B \in \{0, 1, 2\}$, generated via $B = \mathbf{1}\{\varepsilon_B < 0.76\} + \mathbf{1}\{\varepsilon_B > 0.92\}$ with $\varepsilon_B \sim \text{Uniform}(0,1)$. Education $E$ and income $I$ are generated using standard normal variables $\varepsilon_E$ and $\varepsilon_I$ with race-dependent means:
\begin{equation}
\begin{aligned}
E &= \max\{0, \varepsilon_E\}, \quad
I = \exp\left\{0.1 \varepsilon_E + \varepsilon_I \right\},
\end{aligned}
\end{equation}
where $\mu_E = \lambda_{E0} + \mathbf{1}\{B=1\}\lambda_{E1} + \mathbf{1}\{B=2\}\lambda_{E2}$ and $\mu_I = \log(\lambda_{I0} + \mathbf{1}\{B=1\} \lambda_{I1} + \mathbf{1}\{B=2\} \lambda_{I2})$.

Loan decisions $Y \in \{0,1\}$ are modeled by:
\begin{equation}
Y = \mathbf{1}\left\{\varepsilon_Y < \operatorname{expit}(\beta_0 + \beta_1 \mathbf{1}\{B=1\} + \beta_2 \mathbf{1}\{B=2\} + \beta_E E + \beta_I I)\right\},
\end{equation}
where $\varepsilon_Y \sim \text{Uniform}(0,1)$. Here, $\lambda_{E\cdot}$ and $\lambda_{I\cdot}$ govern group-level disparities in education and income, while $\beta_1$ and $\beta_2$ encode the direct influence of race on loan approvals.

\section{Explanation for EO Fairness of \textbf{OB} and \textbf{SOB}}
\label{app:eo_fair}
We explain why \texttt{FTU} achieves EO-fairness and why our methods, \texttt{OB} and \texttt{SOB}, also yield low EO-metric values. The key idea is captured in the following lemma (see~\cite{wang2019equal} for proof):

\begin{lemma}[EO \( \Leftrightarrow \) No \( B \rightarrow \widehat{Y} \)]
Assume the causal graph in Fig.~\ref{fig:sub1}. A decision rule \( \widehat{Y} \) satisfies equal opportunity over \( S \) if and only if there is no causal path from \( B \) to \( \widehat{Y} \).
\end{lemma}

This immediately implies EO-fairness for \texttt{FTU}, where \( \widehat{Y}^{FTU}(a, b) = \widehat{Y}^{FTU}(a) \) (Fig.~\ref{fig:sub2}). Similarly, \texttt{OB} and \texttt{SOB} use predictors on \( \tilde{A} \), the minimal modification of \( A \) that is orthogonal to \( B \) (Figs.~\ref{fig:illustration_c},~\ref{fig:sub3}). This blocks the \( B \rightarrow \widehat{Y} \) path, promoting EO-fairness while preserving utility, as supported by Theorem~\ref{thm:elliptical}.

\section{Discretizing Continuous Sensitive Attributes for Discrete Methods}
\label{app:discretization}
To evaluate fairness methods originally designed for categorical sensitive attributes, we discretize continuous variables (e.g., income, racial composition) into 10 equal-frequency bins, treating each bin as a distinct group. This enables the application of discrete methods such as \texttt{FLAP-o}, \texttt{FLAP-m}, \texttt{FL}, \texttt{EO}, and \texttt{AA}. Table~\ref{tab:discrete_methods_perf} compares these methods with our \texttt{OB} and \texttt{SOB} approaches on the Crime and Synthetic datasets. While discretization makes evaluation feasible, discrete methods often yield worse fairness–accuracy tradeoffs. In contrast, \texttt{OB} and \texttt{SOB}, which are explicitly designed for continuous sensitive attributes, consistently achieve lower counterfactual unfairness with minimal loss in accuracy.

\begin{table*}[!t]
\centering
\caption{Performance of Discrete Fairness Methods on Datasets with Continuous Sensitive Attributes (After Discretization). OB/SOB Results Are Included for Comparison.}
\label{tab:discrete_methods_perf}
\resizebox{\textwidth}{!}{%
\begin{tabular}{llccccccc}
\toprule
\textbf{Dataset} & \textbf{Metric} & \texttt{FLAP-o} & \texttt{FLAP-m} & \texttt{FL} & \texttt{EO} & \texttt{AA} & \texttt{OB} & \texttt{SOB} \\
\midrule
\multirow{2}{*}{Crime} 
& Test MSE & 0.4644 $\pm$ 0.0194 & 0.5414 $\pm$ 0.0323 & 0.5187 $\pm$ 0.0401 & \textbf{0.4170} $\pm$ 0.0291 & \textbf{0.4170} $\pm$ 0.0291 & 0.4534 $\pm$ 0.0110 & 0.4491 $\pm$ 0.0155 \\
& CF Score & \underline{0.0911} $\pm$ 0.0727 & 0.1043 $\pm$ 0.0679 & \textbf{0.0790} $\pm$ 0.0522 & 0.2485 $\pm$ 0.0881 & 0.1526 $\pm$ 0.0393 & 0.1047 $\pm$ 0.0373 & 0.1051 $\pm$ 0.0328 \\
\midrule
\multirow{2}{*}{Synthetic} 
& Test MSE & 0.0148 $\pm$ 0.0008 & 0.0440 $\pm$ 0.0020 & 0.0155 $\pm$ 0.0004 & \textbf{0.0001} $\pm$ 0.0001 & \textbf{0.0001} $\pm$ 0.0001 & 0.0054 $\pm$ 0.0001 & 0.0054 $\pm$ 0.0002 \\
& CF Score & 0.1446 $\pm$ 0.0042 & 0.1447 $\pm$ 0.0036 & 0.1522 $\pm$ 0.0042 & 0.1387 $\pm$ 0.0002 & 0.1386 $\pm$ 0.0002 & \underline{0.1309} $\pm$ 0.0023 & \textbf{0.1296} $\pm$ 0.0020 \\
\bottomrule
\end{tabular}
}
\end{table*}

\end{document}